%% file: main.tex
\newif\ifisarxiv
\documentclass{article}

\ifisarxiv
\usepackage{fullpage}
\input{package_and_def}
\title{Boundary thickness and robustness in learning models}
\date{}

\author{%
  Yaoqing Yang, Rajiv Khanna, Yaodong Yu, Amir Gholami, \\Kurt Keutzer, Joseph E. Gonzalez, Kannan Ramchandran, Michael W. Mahoney \\
  University of California, Berkeley\\
  \texttt{  \{yqyang, rajivak, yaodong\_yu, amirgh, keutzer, jegonzal, kannanr,}\\
  \texttt{mahoneymw\}@berkeley.edu } \\
}

\else
\usepackage[nonatbib,final]{neurips_2020}
\input{package_and_def}
\title{Boundary thickness and robustness \\ in learning models}
\author{%
  Yaoqing Yang, Rajiv Khanna, Yaodong Yu, Amir Gholami,\\ \bf{Kurt Keutzer, Joseph E. Gonzalez, Kannan Ramchandran, Michael W. Mahoney} \\%\thanks{} \\
  University of California, Berkeley\\
  Berkeley, CA 94720 \\
  \texttt{  \{yqyang, rajivak, yaodong\_yu, amirgh, keutzer, jegonzal, kannanr,}\\
  \texttt{mahoneymw\}@berkeley.edu } \\
}

\fi

\begin{document}
\maketitle

\input{Abstract}
\input{Introduction}

\input{Section2_boundary_thickness}
\input{Section3_thickness_and_robustness}
\input{Section4_applications}

\input{Conclusions}

\ifisarxiv
\else
\input{Broad_impact}
\fi

\bibliography{reference}
\bibliographystyle{ieeetr}

\clearpage
\newpage

\appendix
\counterwithin{figure}{section}
\counterwithin{table}{section}

\ifisarxiv
\begin{center}
\huge
    Appendix
\end{center}
\else
\begin{center}
\LARGE
    Supplementary Materials of ``Boundary Thickness and Robustness in Learning Models''
\end{center}
\fi

\input{Appendix_mixup_theory}
\input{Appendix_proofs}
\input{Appendix_chess_board}
\input{Appendix_boundary_tilting}

\input{Appendix_non_adv}
\input{Appendix_adv_training}
\input{Appendix_noisy_mixup}

\input{Appendix_non_robust_feature}

\end{document}

%% file: package_and_def.tex
\usepackage[utf8]{inputenc} % allow utf-8 input
\usepackage[T1]{fontenc}    % use 8-bit T1 fonts
\usepackage{hyperref}       % hyperlinks
\usepackage{url}            % simple URL typesetting
\usepackage{array, booktabs, ragged2e}
\usepackage{booktabs}       % professional-quality tables
\usepackage{amsfonts}       % blackboard math symbols
\usepackage{nicefrac}       % compact symbols for 1/2, etc.
\usepackage{microtype}      % microtypography
\usepackage{xcolor}
\usepackage{cite}
\usepackage{colortbl}
\usepackage{wrapfig}

\newcommand\ga{ \rowcolor{gray!0}}
\newcommand\gb{ \rowcolor{gray!15}}

\usepackage{microtype}
\usepackage{graphicx}
\usepackage{tabularx}
\usepackage{array}
\usepackage{subcaption}
\usepackage{booktabs} % for professional tables
\usepackage{comment}
\usepackage{amsfonts}
\usepackage{amsmath}
\usepackage{enumitem}
\usepackage{tabularx}
\usepackage[export]{adjustbox}
\usepackage{amsthm}
\newtheorem{definition}{Definition}

\newtheorem{example}{Example}

\newtheorem{proposition}{Proposition}[section]
\newtheorem{lemma}{Lemma}[section]

\newtheorem{remark}{Remark}

\usepackage{hyperref}
 

%% file: Abstract.tex
\begin{abstract}

Robustness of machine learning models to various adversarial and non-adversarial corruptions continues to be of interest.
In this paper, we introduce the notion of the boundary thickness of a classifier, and we describe its connection with and usefulness for model robustness. 
Thick decision boundaries lead to improved performance, while thin decision boundaries lead to overfitting (e.g., measured by the robust generalization gap between training and testing) and lower robustness.
We show that a thicker boundary helps improve robustness against adversarial examples (e.g., improving the robust test accuracy of adversarial training) as well as so-called out-of-distribution (OOD) transforms, and we show that many commonly-used regularization and data augmentation procedures can increase boundary thickness.
On the theoretical side, we establish that maximizing boundary thickness during training is akin to the so-called mixup training. 
Using these observations, we show that noise-augmentation on mixup training further increases boundary thickness, thereby combating vulnerability to various forms of adversarial attacks and OOD transforms.
We can also show that the performance improvement in several lines of recent work happens in conjunction with a thicker boundary.
\end{abstract}

%% file: Introduction.tex
\section{Introduction}

Recent work has re-highlighted the importance of various forms of robustness of machine learning models.
For example, it is by now well known that by modifying natural images with barely-visible perturbations, one can get neural networks to misclassify images~\cite{goodfellow2014explaining,nguyen2015deep,moosavi2016deepfool,carlini2017towards}. 
Researchers have come to call these slightly-but-adversarially perturbed images \emph{adversarial examples}. 
As another example, it has become well-known that, even aside from such worst-case adversarial examples, neural networks are also vulnerable to so-called \emph{out-of-distribution (OOD) transforms}~\cite{hendrycks2019benchmarking}, i.e., those which contain common corruptions and perturbations that are frequently encountered in natural images. 
These topics have received interest because they provide visually-compelling examples that expose an inherent lack of stability/robustness in these already hard-to-interpret models~\cite{madry2017towards,zhang2019theoretically,cohen2019certified,hendrycks2019augmix,papernot2016distillation,athalye2018obfuscated,tramer2018ensemble}, but of course similar concerns arise in other less visually-compelling situations. 

In this paper, we study neural network robustness through the lens of what we will call \emph{boundary thickness}, a new and intuitive concept that we introduce. 
Boundary thickness can be considered a generalization of the standard margin, used in max-margin type learning~\cite{elsayed2018large,bartlett2017spectrally,sokolic2017robust}. 
Intuitively speaking, the boundary thickness of a classifier measures the expected distance to travel along line segments between different classes across a decision boundary. 
We show that thick decision boundaries have a regularization effect that improves robustness, while thin decision boundaries lead to overfitting and reduced robustness.
We also illustrate that the performance improvement in several lines of recent work happens in conjunction with a thicker boundary, suggesting the utility of this notion more~generally.

More specifically, for adversarial robustness, we show that five commonly used ways to improve robustness can increase boundary thickness and reduce the robust generalization gap (which is the difference between robust training accuracy and robust test accuracy) during adversarial training. 
We also show that trained networks with thick decision boundaries tend to be more robust against OOD transforms. 
We focus on \emph{mixup training}~\cite{zhang2017mixup}, a recently-described regularization technique that involves training on data that have been augmented with pseudo-data points that are convex combinations of the true data points. 
We show that mixup improves robustness to OOD transforms, while at the same time achieving a thicker decision boundary. 
In fact, the boundary thickness can be understood as a dual concept to the mixup training objective, in the sense that the former is maximized as a result of minimizing the mixup loss. 
In contrast to measures like margin, boundary thickness is easy to measure, and (as we observe through counter examples) boundary thickness can differentiate neural networks of different robust generalization gap, while margin cannot.

For those interested primarily in training, our observations also lead to novel training procedures. 
Specifically, we design and study a novel noise-augmented extension of mixup, referred to as \emph{noisy mixup}, which augments the data through a mixup with random noise, to improve robustness to image imperfections. 
We show that noisy mixup thickens the boundary, and thus it significantly improves robustness, including black/white-box adversarial attacks, as well as OOD transforms. 

In more detail, here is a summary of our main contributions. 

\begin{itemize}
\item 
We introduce the concept of boundary thickness (Section \ref{sec:boundary_thickness_intro}), and we illustrate its connection to various existing concepts, including showing that as a special case it reduces to~margin.
\item 
We demonstrate empirically that a thin decision boundary leads to poor adversarial robustness as well as poor OOD robustness (Section \ref{sec:boundary_regularization}), and we evaluate the effect of model adjustments that affect boundary thickness. 
In particular, we show that five commonly used regularization and data augmentation schemes ($\ell_1$ regularization, $\ell_2$ regularization, large learning rate \cite{li2019towards}, early stopping, and cutout \cite{devries2017improved}) all increase boundary thickness and reduce overfitting of adversarially trained models (measured by the robust accuracy gap between training and testing). 
We also show that boundary thickness outperforms margin as a metric in measuring the robust generalization gap.
\item 
We show that our new insights on boundary thickness make way for the design of new robust training schemes (Section \ref{sec:applications}).
In particular, we designed a noise-augmentation training scheme that we call noisy mixup to increase boundary thickness and improve the robust test accuracy of mixup for  both adversarial examples and OOD transforms. 
We also show that mixup achieves the minimax decision boundary thickness, providing a theoretical justification for both mixup and noisy mixup.
\end{itemize} 

\noindent
Overall, our main conclusion is the following.

\begin{quote}
\emph{%
Boundary thickness is a reliable and easy-to-measure metric that is associated with model robustness, and training a neural network while ensuring a thick boundary can improve robustness in various ways that have received attention recently. 
}
\end{quote}

In order that our results can be reproduced and extended, we have open-sourced our code.\footnote{https://github.com/nsfzyzz/boundary\_thickness}

\paragraph{Related work.}

Both adversarial robustness \cite{goodfellow2014explaining,nguyen2015deep,moosavi2016deepfool,carlini2017towards,madry2017towards,zhang2019theoretically,cohen2019certified,athalye2018obfuscated} and OOD robustness \cite{hendrycks2019augmix,hendrycks2019benchmarking,yin2019fourier,liang2018enhancing,snoek2019can} have been well-studied in the literature. 
From a geometric perspective, one expects robustness of a machine learning model to relate to its decision boundary. 
In \cite{goodfellow2014explaining}, the authors claim that adversarial examples arise from the linear nature of neural networks, hinting at the relationship between decision boundary and robustness. 
In \cite{tanay2016boundary}, the authors provide the different explanation that the decision boundary is not necessarily linear, but it tends to lie close to the ``data sub-manifold.'' 
This explanation is supported by the idea that cross-entropy loss leads to poor margins \cite{nar2018cross}. 
Some other works also study the connection between geometric properties of a decision boundary and the robustness of the model, e.g., on the boundary curvature \cite{moosavi2017universal,fawzi2017robustness}. 
The paper \cite{zhang2020attacks} uses early-stopped PGD to improve natural accuracy, and it discusses how robust loss changes the decision boundary.
Another related line of recent work points out that the inductive bias of neural networks towards ``simple functions'' may have a negative effect on network robustness \cite{nakkiran2019adversarial}, although being useful to explain generalization \cite{de2019random,valle2018deep}.
These papers support our observation that natural training tends to generate simple, thin, but easy-to-attack decision boundaries, while avoiding thicker more robust boundaries.

Mixup is a regularization technique introduced by \cite{zhang2017mixup}. In mixup, each sample is obtained from two samples $x_1$ and $x_2$ using a convex combination $x = t x_1 + (1-t) x_2$, with some random $t \in [0,1]$. The label $y$ is similarly obtained by a convex combination.
Mixup and some other data augmentation schemes have been successfully applied to improve robustness.
For instance, \cite{lamb2019interpolated} uses mixup to interpolate adversarial examples to improve adversarial robustness. 
The authors in \cite{lee2020adversarial} combine mixup and label smoothing on adversarial examples to reduce the variance of feature representation and improve robust generalization. 
\cite{lopes2019improving} uses interpolation between Gaussian noise and Cutout \cite{devries2017improved} to improve OOD robustness.
The authors in \cite{hendrycks2019augmix} mix images augmented with various forms of transformations with a smoothness training objective to improve OOD robustness.
Compared to these prior works, the extended mixup with simple noise-image augmentation studied in our paper is motivated from the perspective of decision boundaries; and it provides a concrete explanation for the performance improvement, as regularization is introduced by a thicker boundary. 
Another recent paper \cite{rice2020overfitting} also shows the importance of reducing overfitting in adversarial training, e.g., using early stopping, which we also demonstrate as one way to increase boundary thickness.

%% file: Section2_boundary_thickness.tex
\section{Boundary Thickness}
\label{sec:boundary_thickness_intro}

In this section, we introduce boundary thickness and discuss its connection with related notions. 

\begin{figure}
    \begin{subfigure}{.18\textwidth}
        \includegraphics[width=.99\linewidth]{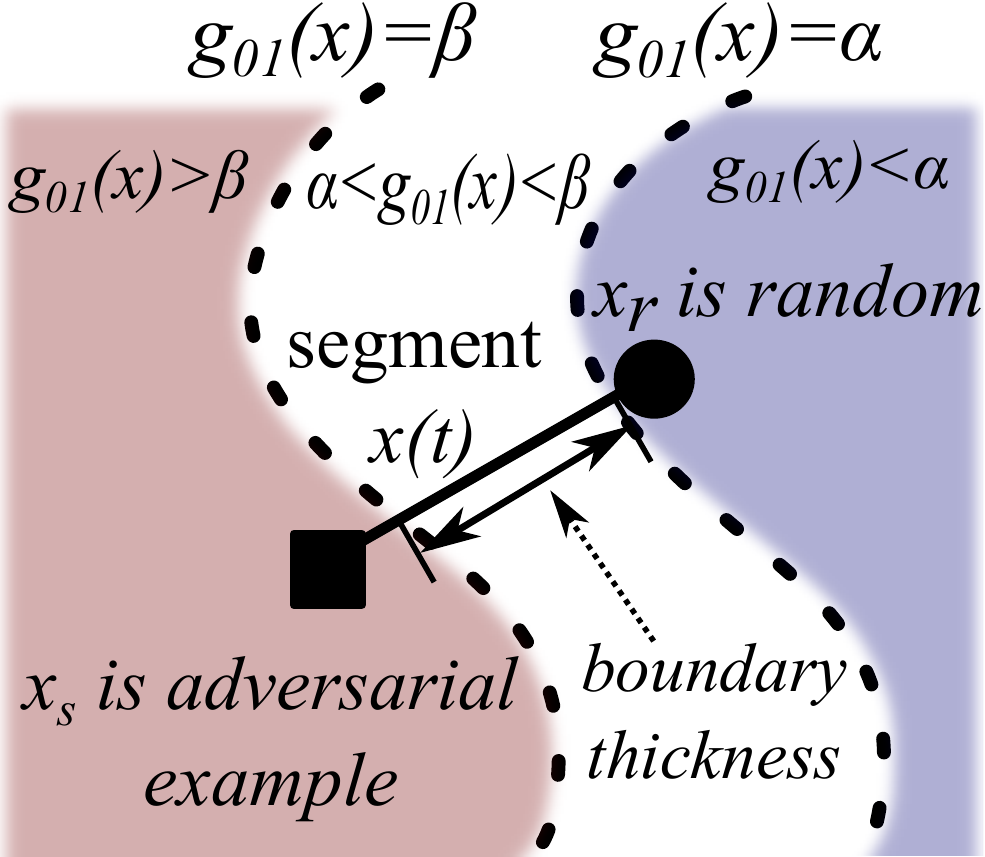}\vspace{-2mm}
        \subcaption{\label{fig:thickness_definition}}
    \end{subfigure}
    \begin{subfigure}{.30\textwidth}
        \vspace{2mm}
        \includegraphics[width=.99\linewidth]{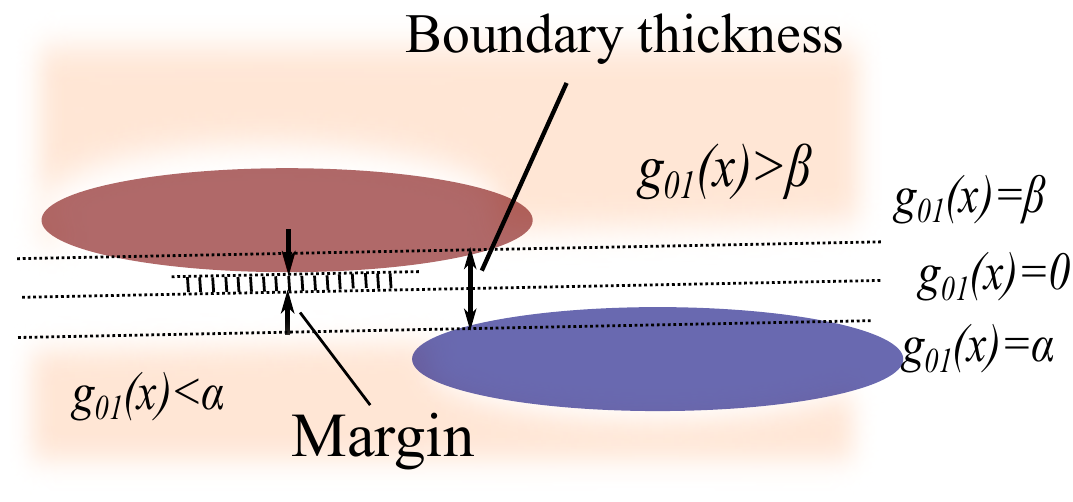}\vspace{-2mm}
        \subcaption{\label{fig:thickness_intuition_1}}
    \end{subfigure}
    \begin{subfigure}{.30\textwidth}
        \vspace{2mm}
        \includegraphics[width=.99\linewidth]{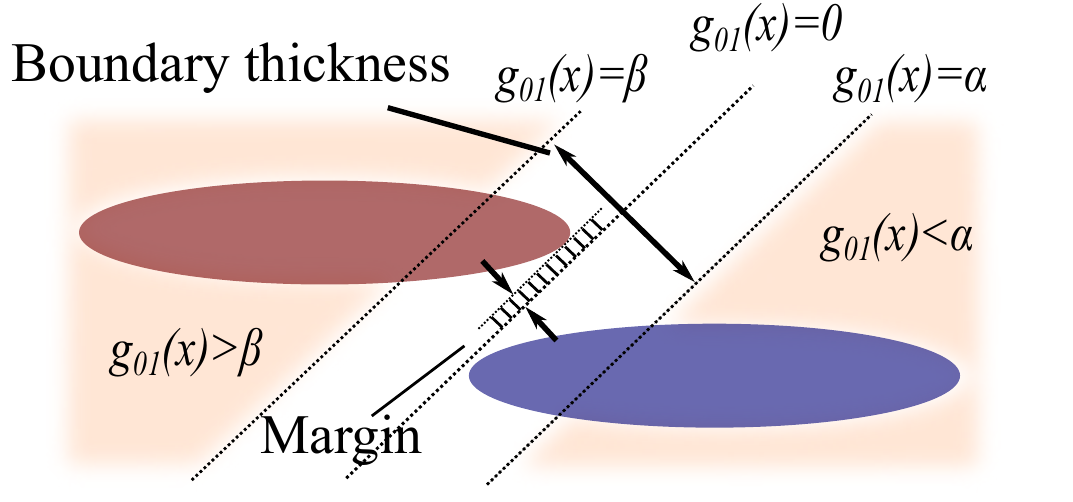}\vspace{-2mm}
        \subcaption{\label{fig:thickness_intuition_2}}
    \end{subfigure}
    \begin{subfigure}{.18\textwidth}
        \includegraphics[width=.99\linewidth]{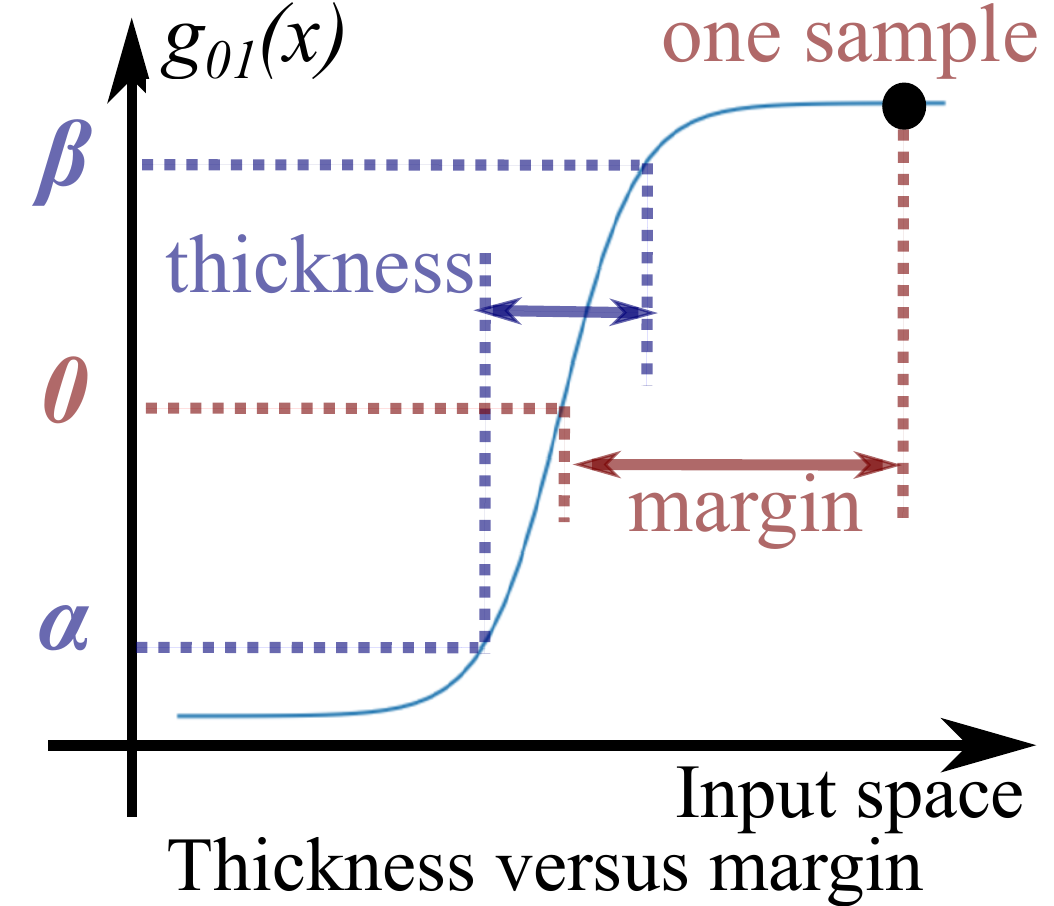}\vspace{-2mm}
        \subcaption{\label{fig:thickness_intuition_3}}
    \end{subfigure}
    \caption{{\bf Main intuition behind boundary thickness.} 
    (a) Boundary thickness measures the gap between two level sets $g_{01}(x)=\alpha$ and $g_{01}(x)=\beta$ along the adversarial direction. 
    (b) A thin boundary easily fits into the narrow space between two different classes, but it is not robust. 
    (c) A thick boundary is harder to achieve a small loss, but it achieves higher robustness. 
    In these two toy settings, the margin is the same, but choosing a thicker boundary leads to better robustness. (d) Illustration of the relationship between boundary thickness and margin.
    }
    \label{fig:max-margin}
\end{figure}

\subsection{Boundary thickness}\label{sec:definition_of_thickness}

Consider a classification problem with $C$ classes on the domain space of data  $\mathcal{X}$. 
Let $f(x): \mathcal{X} \to [0,1]^C$ be the prediction function, so that $f(x)_i$ for class $i \in [C]$ represents the posterior probability $\Pr(y=i|x)$, where $(x,y)$ represents a feature vector and response pair. 
Clearly, $\sum_{i=0}^{C-1} f(x)_i = 1, \forall x \in \mathcal{X}$. For neural networks, the function $f(x)$ is the output of the softmax layer.
In the following definition, we quantify the thickness of a decision boundary by measuring the posterior probability difference $ g_{ij}(x) = f(x)_i-f(x)_j$ on line segments connecting pairs of points $(x_r, x_s)\in \mathcal{X}$ (where $x_r, x_s$ are \emph{not} restricted to the training set). 

\begin{definition}[Boundary Thickness]
\label{def:boundary_thickness}
For $\alpha, \beta \in (-1,1)$ and a distribution $p$ over pairs of points $(x_r, x_s) \sim p$, let the predicted labels of $x_r$ and $x_s$ be $i$ and $j$ respectively. 
Then, the boundary thickness of a prediction function $f(\cdot)$ is
\begin{equation}\label{eqn:def_ij}
    \Theta(f, \alpha, \beta, p) = \mathbb{E}_{(x_r, x_s) \sim p} \left[\|x_r-x_s\| \int_0^1 \mathbf{I} \{\alpha < g_{ij}(x(t)) <\beta\}dt  \right],
\end{equation}
where $ g_{ij}(x) = f(x)_i-f(x)_j$, $\mathbf{I}\{\cdot\}$ is the indicator function, and $x(t) = t x_r + (1-t) x_s, t\in [0,1]$.
\end{definition}

Intuitively, boundary thickness captures the distance between two level sets $g_{ij}(x)=\alpha$ and $g_{ij}(x)=\beta$ by measuring the expected gap on random line segments in $\mathcal{X}$. See Figure \ref{fig:thickness_definition}. Note that in addition to the two constants, $\alpha$ and $\beta$, Definition~\ref{def:boundary_thickness} of boundary thickness requires one to specify a distribution $p$ to choose pairs of points $(x_r, x_s)$. 
We show that specific instances of $p$ recovers margin (Section \ref{sec:dif_margin}) and mixup regularization (Section \ref{sec:mixup_theory}). For the rest of the paper, we set $p$ as follows. Choose $x_r$ uniformly at random from the training set. Denote $i$ its predicted label. Then, choose $x_s$ to be an $\ell_2$ adversarial example generated by attacking $x_r$ to a random target class $j\neq i$. We first look at a simple example on linear classifiers to illustrate the concept.

\begin{example}[Binary Linear Classifier]
\normalfont
Consider a binary linear classifier, with weights $w$ and bias $b$. 
The prediction score vector is $f(x) = [f(x)_0, f(x)_1] = [\sigma(w^\top x + b), 1-\sigma(w^\top x + b)]$, where $\sigma(\cdot): \mathbb{R}\to [0,1]$ is the sigmoid function. 
In this case, measuring thickness in the $\ell_2$ adversarial direction means that we choose $x_r$ and $x_s$ such that $x_s - x_r = cw, c\in \mathbb{R}$. 
In the following proposition, we quantify the boundary thickness for a binary linear classifier.
(See Section \ref{proof:linear} for the~proof.)

\begin{proposition}[Boundary Thickness of Binary Linear Classifier] 
\label{prop:linear}
Let $\tilde{g}(\cdot) := 2\sigma(\cdot) - 1$. 
If $[\alpha,\beta]\subset [g_{01}(x_r), g_{01}(x_s)]$, then the thickness of the binary linear classifier is given by:
\begin{equation}\label{eqn:linear_thickness_final}
    \Theta(f, \alpha, \beta) = (\tilde{g}^{-1}(\beta) - \tilde{g}^{-1}(\alpha))/\|w\|.
\end{equation}
\end{proposition}

Note that $x_s$ and $x_r$ should be chosen such that the condition $[\alpha,\beta]\subset[g_{01}(x_r), g_{01}(x_s)]$ in Proposition \ref{prop:linear} is satisfied. 
Otherwise, for linear classifiers, the segment between $x_r$ and $x_s$ is not long enough to span the gap between the two level sets $g_{01}(x) = \alpha$ and $g_{01}(x) = \beta$ and cannot simultaneously intersect the two level~sets.

\end{example}

For a pictorial illustration of why boundary thickness should be related to robustness and why a thicker boundary should be desirable, see Figures~\ref{fig:thickness_intuition_1} and~\ref{fig:thickness_intuition_2}. 
The three curves in each figure represent the three level sets $g_{01}(x)=\alpha$, $g_{01}(x)=0$, and $g_{01}(x)=\beta$. 
The thinner boundary in Figure~\ref{fig:thickness_intuition_1} easily fits the narrow space between two different classes, but it is easier to attack. 
The thicker boundary in Figure~\ref{fig:thickness_intuition_2}, however, is harder to fit data with a small loss, but it is also more robust and harder to attack. 
To further justify the intuition, we provide an additional example in Section \ref{sec:chess_board}.
Note that the intuition discussed here is reminiscent of max margin optimization, but it is in fact more general.
In Section,~\ref{sec:dif_margin}, we highlight differences between the two concepts (and later, in Section~\ref{sec:thickness_beats_margin}, we also show that margin is not a particularly good indicator of robust~performance).

\subsection{Boundary thickness generalizes margin}\label{sec:dif_margin}

We first show that boundary thickness reduces to margin in the special case of binary linear SVM. We then extend this result to general classifiers.

\begin{example}[Support Vector Machines]
\label{ex:svm}
\normalfont
As an application of Proposition \ref{prop:linear}, we can compute the boundary thickness of a binary SVM, which we show is equal to the margin. 
Suppose we choose $\alpha$ and $\beta$ to be the values of $\tilde{g}(u)$ evaluated at two support vectors, i.e., at $u = w^\top x +b = -1$ and $u = w^\top x +b = 1$. Then, $\tilde{g}^{-1}(\alpha) = -1$ and $\tilde{g}^{-1}(\beta) = 1$. 
Thus, from \eqref{eqn:linear_thickness_final}, we obtain $(\tilde{g}^{-1}(\beta) - \tilde{g}^{-1}(\alpha))/\|w\| = 2/\|w\|$, which is the (input-space) margin of an SVM. 
\end{example}

We can also show that the reduction to margin applies to more general classifiers. Let $S(i,j) = \{x\in \mathcal{X}: f(x)_i = f(x)_j\}$ denote the decision boundary between classes $i$ and $j$. 
The (input-space) margin~\cite{elsayed2018large} of $f$ on a dataset $\mathcal{D} = \{x_k\}_{k=0}^{n-1}$ is defined as
\begin{equation}\label{eqn:margin}
    \textit{Margin}(\mathcal{D}, f) = \min_k\min_{j\neq y_k}\| x_k - \text{Proj}(x_k, j) \|,
\end{equation}
where $\text{Proj}(x, j) = \arg\min_{x' \in S(i_x,j)} \| x' - x\|$ is the projection onto the decision boundary $S(i_x,j)$. 
See Figure~\ref{fig:thickness_intuition_1} and~\ref{fig:thickness_intuition_2}.

Boundary thickness for the case when $\alpha=0$, $\beta = 1$, and when $p$ is so chosen that $x_s$ is the projection $\text{Proj}(x_r, j)$ for the worst case class $j$, reduces to margin. 
See Figure \ref{fig:thickness_intuition_3} for an illustration of this relationship for a two-class problem.
Note that the left hand side of \eqref{eqn:thickness_reduces_to_margin} is a ``worst-case'' version of the boundary thickness in \eqref{eqn:def_ij}. 
This can be formalized in the following proposition.
(See Section \ref{proof:margin} for the proof.)
\begin{proposition}[Margin is a Special Case of Boundary Thickness]
\label{prop:margin}
Choose $x_r$ as an arbitrary point in the dataset $\mathcal{D} = \{x_k\}_{k=0}^{n-1}$, with predicted label $i = \arg\max_l f(x)_l$. For another class $j\neq i$, choose $x_s = \text{Proj}(x_r, j)$. Then,
\begin{equation}\label{eqn:thickness_reduces_to_margin}
\min_{x_r}\min_{j\neq i}\|x_r-x_s\| \int_0^1 \mathbf{I} \{\alpha < g_{ij}(x(t)) <\beta\}dt =  \text{Margin}(\mathcal{D}, f).
\end{equation}
\end{proposition}

\begin{remark}[Margin versus Thickness as a Metric]
\normalfont
It is often impractical to compute the margin for general nonlinear functions. 
On the other hand, as we illustrate below, measuring boundary thickness is straightforward. 
As noted by \cite{zhang2017mixup}, using mixup tends to make a decision boundary more ``linear,'' which helps to reduce unnecessary oscillations in the boundary. As we show in Section \ref{sec:noisy_mixup}, mixup effectively makes the boundary thicker. This effect is not directly achievable by increasing margin.
\end{remark}

\subsection{A thick boundary mitigates boundary tilting}\label{sec:boundary_tilting_theory}

Boundary tilting was introduced by \cite{tanay2016boundary} to capture the idea that for many neural networks the decision boundary ``tilts'' away from the max-margin solution and instead leans towards a ``data sub-manifold,'' which then makes the model less robust. 
Define the cosine similarity between two vectors $a$ and $b$ as: 
\begin{equation}
    \text{Cosine Similarity}(a, b) = |a^\top b|/(\|a\|_2 \cdot \|b\|_2).
\end{equation}

For a dataset $D = \{(x_i, y_i)\}_{i\in I}$, in which the two classes $\{x_i\in D: y_i=1\}$ and $\{x_i\in D: y_i=-1\}$ are linearly separable, boundary tilting can be defined as the worse-case cosine similarity between a classifier $w$ and the hard-SVM solution $w^* := \arg\min \|v\|_2 \text{ s.t. }y_i v^\top x_i\ge 1, \forall i \in I$:
\begin{equation}\label{eqn:worst_tilting}
    T(u) := \min_{v\text{ s.t. }\|v\|_2 = u \text{ and } y_i v^\top x_i\ge 1, \forall i } \text{Cosine Similarity}(v, w^*),
\end{equation}
which is a function of the $u$ in the $\ell_2$ constraint $\|v\|_2 = u$. 
In the following proposition, we formalize the idea that boundary thickness tends to mitigate tilting. 
(See Section \ref{proof:tilting} for the proof; and see also Figure \ref{fig:thickness_intuition_1}.)

\begin{proposition}[A Thick Boundary Mitigates Boundary Tilting]\label{prop:tilting}
The worst-case boundary tilting $T(u)$ is a non-increasing function of $u$.
\end{proposition}

A smaller cosine similarity between the $w$ and the SVM solution $w^*$ corresponds to more tilting. 
In other words, $T(u)$ achieves the maximum when there is no tilting.
From Proposition \ref{prop:linear}, for a linear classifier, we know that $u=\|w\|_2$ is inversely proportional to thickness. 
Thus, $T(u)$ is a non-decreasing function in thickness. 
That is, a thicker boundary leads to a larger $T(u)$, which means the worst-case boundary tilting is mitigated. 
We demonstrate Proposition~\ref{prop:tilting} on the more general nonlinear classifiers in Section~\ref{sec:thick_boundary_reduces_tilting}.

%% file: Section3_thickness_and_robustness.tex
\section{Boundary Thickness and Robustness}\label{sec:boundary_regularization}

In this section, we measure the change in boundary thickness by slightly altering the training algorithm in various ways, and we illustrate the corresponding change in robust accuracy. 
We show that across many different training schemes, boundary thickness corresponds strongly with model robustness. 
We observe this correspondence for both non-adversarial as well as adversarial training. 
We also present a use case illustrating why using boundary thickness rather than margin as a metric for robustness is useful. 
More specifically, we show that a thicker boundary reduces overfitting in adversarial training, while margin is unable to differentiate different levels of overfitting. 

\subsection{Non-adversarial training}\label{sec:regularization}

\begin{figure}
    \centering
    \begin{subfigure}{0.98\textwidth}
    \includegraphics[width=.98\textwidth]{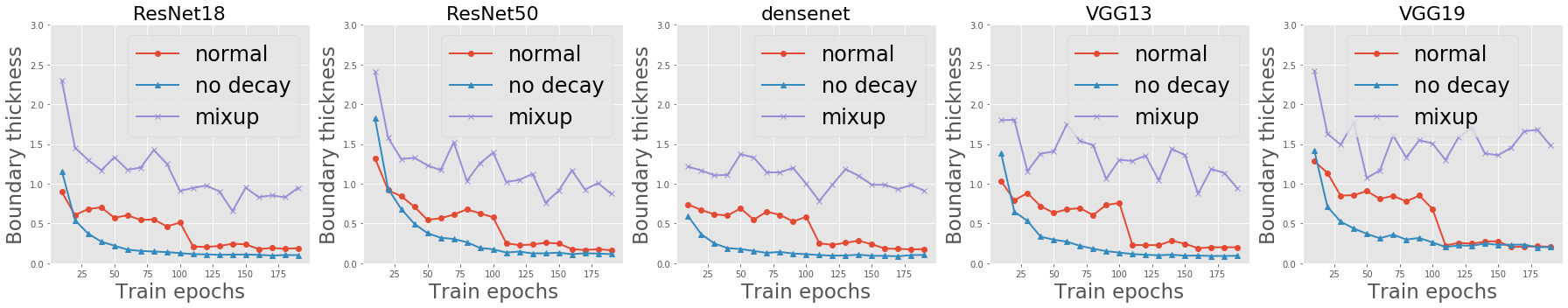}
    \subcaption{Thickness: mixup $>$ normal training $>$ training without weight decay. After learning rate decays (at both epoch 100 and 150), decision boundaries get thinner. %Note that the neural networks on the second row are much wider than the first row.
    }
    \label{fig:thickness_adv_direction}
    \end{subfigure}
    \begin{subfigure}{0.98\textwidth}
    \centering
    \includegraphics[width=.98\textwidth]{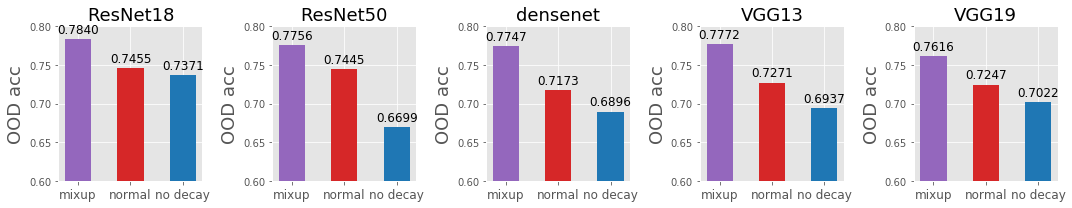}
    \subcaption{OOD robustness: mixup $>$ normal training $>$ training without weight decay. Compare with Figure \ref{fig:thickness_adv_direction} to see that mixup increases thickness, while training without weight decay reduces thickness.}
    \label{fig:OOD_measurement}
    \end{subfigure}
    \caption{{\bf OOD robustness and thickness.} OOD robustness improves with increasing boundary thickness for a variety of neural networks trained on CIFAR10. ``Normal'' means training with standard weight decay 5e-4, and ``no decay'' means training without weight decay. Mixup uses the recommended weight decay 1e-4.\label{fig:OOD_results}
    }
    
\end{figure}

Here, we compare the boundary thicknesses and robustness of models trained with three different schemes on CIFAR10 \cite{cifar10}, namely training without weight decay, training with standard weight decay, and mixup training \cite{zhang2017mixup}. 
Note that these three training schemes impose increasingly stronger regularization.
We use different neural networks, including ResNets \cite{resnet}, VGGs \cite{vgg}, and DenseNet \cite{densenet}.\footnote{The models in Figure \ref{fig:OOD_results} are from https://github.com/kuangliu/pytorch-cifar/blob/master/models/resnet.py. }  
All models are trained with the same initial learning rate of 0.1. At both epoch 100 and 150, we reduce the current learning rate by a factor of 10.
The thickness of the decision boundary is measured as described in Section~\ref{sec:boundary_thickness_intro} with $\alpha=0$ and $\beta=0.75$. When measuring thickness on the adversarial direction, we use an $\ell_2$ PGD-20 attack with size 1.0 and step size 0.2. We report the average thickness obtained by repeated runs on 320 random samples, i.e., 320 random samples with their adversarial examples.
The results are shown in Figure~\ref{fig:thickness_adv_direction}.

From Figure \ref{fig:thickness_adv_direction}, we see that the thickness of mixup is larger than that of training with standard weight decay, which is in turn larger than that of training without weight decay. From the thickness drop at epochs 100 and 150, 
we conclude that learning rate decay reduces the boundary thickness. 
Then, we compare the OOD robustness for the three training procedures on the same set of trained networks from the last epoch. 
For OOD transforms, we follow the setup in \cite{hendrycks2019using}, and we evaluate the trained neural networks on CIFAR10-C, which contains 15 different types of corruptions, including noise, blur, weather, and digital corruption. 
From Figure \ref{fig:OOD_measurement}, we see that the OOD robustness corresponds to boundary thickness across different training schemes for all the tested networks. 

See Section \ref{sec:details_measure_thickness} for more details on the experiment. 
See Section \ref{sec:compare_different_thickness} for a discussion of why the adversarial direction is preferred in measuring thickness. 
See Section \ref{sec:ablation_non_adv} for a thorough ablation study of the hyper-parameters, such as $\alpha$ and $\beta$, and on the results of two other datasets, namely CIFAR100 and SVHN \cite{SVHN}. 
See Section \ref{sec:visualization_mixup_boundary} for a visualization of the decision boundaries of normal versus mixup training, which shows that mixup indeed achieves a thicker boundary.

\subsection{Adversarial training}
\label{sec:adv_training_thickness}

Here, we compare the boundary thickness of adversarially trained neural networks in different training settings. More specifically, we study the effect of five regularization and data augmentation schemes, including large initial learning rate, $\ell_2$ regularization (weight decay), $\ell_1$ regularization, early stopping, and cutout. We choose a variety of hyper-parameters and plot the robust test accuracy versus thickness. We only choose hyper-parameters such that the natural training accuracy is larger than 90\%. We also plot the robust generalization gap versus thickness. See Figure \ref{fig:adv_thickness_all_1} and Figure \ref{fig:adv_thickness_all_2}. 
We again observe a similar correspondence---the robust generalization gap reduces with increasing thickness. 
\begin{figure}
    \centering
     \begin{subfigure}{0.325\textwidth}
        \includegraphics[width=1.0\linewidth]{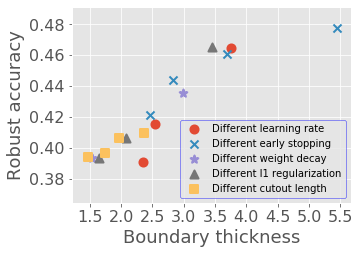}\vspace{-2mm}
        \subcaption{\label{fig:adv_thickness_all_1}}
    \end{subfigure}
    \begin{subfigure}{0.325\textwidth}
        \includegraphics[width=1.0\linewidth]{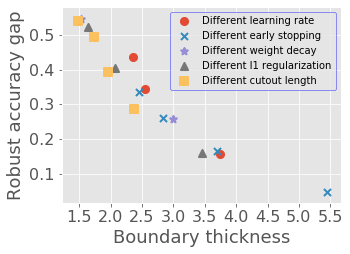}\vspace{-2mm}
        \subcaption{\label{fig:adv_thickness_all_2}}
    \end{subfigure}
    \begin{subfigure}{0.33\textwidth}
        \includegraphics[width=1.0\textwidth]{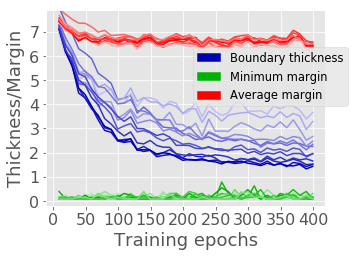}\vspace{-2mm}
        \subcaption{\label{fig:adv_thickness_all_3}}
    \end{subfigure}
    \caption{{\bf Adversarial robustness and thickness.} (a) Increasing boundary thickness improves robust accuracy in adversarial training. (b) Increasing boundary thickness reduces overfitting (measured by robust accuracy gap between training and testing). (c) Thickness can differentiate models of different robust levels (dark to light blue), while margin cannot (dark to light red and dark to light green). Results are obtained for ResNet-18 trained on CIFAR10.}
    \label{fig:adv_thickness_all}
\end{figure}

{\bf Experimental details.} 
In our experiments, we train a ResNet18 on CIFAR10. 
In each set of experiments, we only change one parameter.
In the standard setting, we follow convention and train with learning rate 0.1, weight decay 5e-4, attack range $\epsilon = 8$ pixels, 10 iterations for each attack, and 2 pixels for the step-size. Then, for each set of experiments, we change one parameter based on the standard setting. For $\ell_2$, $\ell_1$ and cutout, we only use one of them at a time to separate their effects. See Section \ref{sec:details_adv_experiment} for the details of these hyper-parameters. Specifically, see Figure \ref{fig:thickness_increases_regularization} which shows that all the five regularization and augmentation schemes increase boundary thickness. We train each model for enough time (400 epochs) to let both the accuracy curves and the boundary thickness stabilize, and to filter out the effect of early stopping. In Section \ref{sec:adv_decay}, we reimplement the whole procedure with the same early stopping at 120 epochs and learning rate decay at epoch 100.
We show that the positive correspondence between robustness and boundary thickness remains the same (see Figure \ref{fig:adv_thickness_all_decay}). 
In Section \ref{sec:ablation_adv_networks}, we provide an ablation study on the hyper-parameters in measuring thickness and again show the same correspondence for the other settings (see Figure \ref{fig:adv_exp_ablation}). 
In Section \ref{sec:margin_vs_thickness_more}, we provide additional analysis on the comparison between boundary thickness and margin.

\subsection{Boundary thickness versus margin}
\label{sec:thickness_beats_margin}

Here, we compare margin versus boundary thickness at differentiating robustness levels. 
See Figure \ref{fig:adv_thickness_all_3}, where we sort the different models shown in Figure \ref{fig:adv_thickness_all_1} by robustness, and we plot their thickness measurements using gradually darker colors. Curves with a darker color represent less robust models.
We see that while boundary thickness correlates well with robustness and hence can differentiate different robustness levels, margin is not able to do this.

From \eqref{eqn:margin}, we see that computing the margin requires computing the projection $\text{Proj}(x_r,j)$, which is intractable for general nonlinear functions. 
Thus, we approximate the margin on the direction of an adversarial attack (which is the projection direction for linear classifiers). 
Another important point here is that we compute the average margin for all samples in addition to the minimum (worst-case) margin in Definition \ref{eqn:margin}.
The minimum margin is almost zero in all cases due to the existence of certain samples that are extremely close to the boundary. 
That is, the standard (widely used) definition of margin performs even worse.

%% file: Section4_applications.tex
\section{Applications of Boundary Thickness}\label{sec:applications}

While training is not our main focus, our insights motivate new training schemes. At the same time, our insights also aid in explaining the robustness phenomena discovered in some contemporary works, when viewed through the connection between thickness and robustness.  

\subsection{Noisy mixup}\label{sec:noisy_mixup}

\begin{table}
\small
\centering
\begin{tabular}{p{1.5cm}p{1.8cm}p{1.4cm}p{1.1cm}p{1.5cm}p{1.1cm}p{1.1cm}p{1.1cm}}
\hline
\hline
\ga Dataset & \multicolumn{1}{c}{Method} & \multicolumn{1}{c}{Clean} & \multicolumn{1}{c}{OOD} & Black-box  & \multicolumn{3}{c}{PGD-20} \\
\hline 
\gb &&&&& 8-pixel & 6-pixel & 4-pixel\\
\hline
\ga CIFAR10 &  \multicolumn{1}{c}{Mixup} & {\bf96.0}$\pm$0.1 &  78.5$\pm$0.4   &  46.3$\pm$1.4 & 2.0$\pm$0.1        &  3.2$\pm$0.1  &  6.3$\pm$0.1 \\
%\cmidrule{2-7}
\gb  & \multicolumn{1}{c}{Noisy mixup} & 94.4$\pm$0.2 &  {\bf 83.6}$\pm$0.3 &  {\bf 78.0}$\pm$1.0 & {\bf 11.7}$\pm$3.3  & {\bf 16.2}$\pm$4.2  & {\bf 25.7}$\pm$5.0 \\
 \hline
\ga CIFAR100 &  \multicolumn{1}{c}{Mixup} & {\bf78.3}$\pm$0.8 &  51.3$\pm$0.4  & 37.3$\pm$1.1 &  0.0$\pm$0.0       &  0.0$\pm$0.0 &  0.1$\pm$0.0\\
%\cmidrule{2-7}
\gb  & \multicolumn{1}{c}{Noisy mixup} & 72.2$\pm$0.3 &  {\bf 52.5}$\pm$0.7 &  {\bf 60.1}$\pm$0.3 & {\bf 1.5}$\pm$0.2  & {\bf 2.6$\pm$0.1} & {\bf 6.7}$\pm$0.9 \\
 \hline
 \hline
\end{tabular}
\caption{{\bf Mixup and noisy mixup.} The robust test accuracy of noisy mixup is significantly higher than ordinary mixup. Results are reported for ResNet-18 and for the best learning rate in [0.1, 0.03, 0.01]. 
\label{tab:noisy_mixup}}
\end{table}

Motivated by the success of mixup~\cite{zhang2017mixup} and our insights into boundary thickness, we introduce and evaluate a training scheme that we call \emph{noisy-mixup}.

{\bf Theoretical justification.}
Before presenting noisy mixup, we strengthen the connection between mixup and boundary thickness by stating that the model which minimizes the mixup loss also achieves optimal boundary thickness in a minimax sense. 
Specifically, we can prove the following:
%\begin{quote}
For a fixed arbitrary integer $c>1$, the model obtained by mixup training achieves the minimax boundary thickness, i.e., $f_\text{mixup}(x) = \arg\max_{f(x)} \min_{(\alpha, \beta)} \Theta(f)$, where the minimum is taken over all possible pairs of $(\alpha,\beta)\in (-1,1)$ such that $\beta-\alpha = 1/c$, and the max is taken over all prediction functions $f$ such that $\sum_i f(x)_i=1$. 
%\end{quote}
See Section~\ref{sec:mixup_theory} for the formal theorem statement and proof.

Ordinary mixup thickens decision boundary by mixing different training samples. 
The idea of noisy mixup, on the other hand, is to thicken the decision boundary between clean samples and arbitrary transformations.
This increases the robust performance on OOD images, for example on images that have been transformed using a noise filter or a rotation. 
Interestingly, pure noise turns out to be good enough to represent such arbitrary transformations. 
So, while the ordinary mixup training obtains one mixup sample $x$ by linearly combining two data samples $x_1$ and $x_2$, in noisy-mixup, one of the combinations of $x_1$ and $x_2$, with some probability $p$, is replaced by an image that consists of random noise. 
The label of the noisy image is ``NONE.'' Specifically, in the CIFAR10 dataset, we let the ``NONE'' class be the $11^\text{th}$ class. Note that this method is different than common noise augmentation because we define a new class of pure noise, and we mix it with ordinary samples.

The comparison between the noisy mixup and ordinary mixup training is shown in Table \ref{tab:noisy_mixup}. 
For OOD accuracy, we follow \cite{hendrycks2019augmix} and use both CIFAR-10C and CIFAR-100C. 
For PGD attack, we use an $\ell_\infty$ attack with 20 steps and with step size being 1/10 of the attack range. We report the results of three different attack ranges, namely 8-pixel, 6-pixel, and 4-pixel.
For black-box attack, we use ResNet-110 to generate the transfer attack. 
The other parameters are the same with the 8-pixel white-box attack. 
For each method and dataset, we run the training procedures with three learning rates (0.01, 0.03, 0.1), each for three times, and we report the mean and standard deviation of the best performing learning rate. 
See Section \ref{sec:details_noisy_mixup} for more details of the~experiment.

\begin{wrapfigure}{r}{0.3\textwidth}
\vspace{-10mm}

  \begin{center}
    \includegraphics[width=0.27\textwidth]{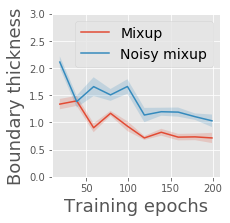}
  \end{center}
  \caption{Noisy miuxp thickens the decision boundary.}
  \label{fig:noisy_mixup}
  
\end{wrapfigure}
From Table \ref{tab:noisy_mixup}, we see that noisy mixup significantly improves the robust accuracy of different types of corruptions. 
Although noisy mixup slightly reduces clean accuracy, the drop of clean accuracy is expected for robust models. For example, we tried adversarial training in the same setting and achieved 57.6\% clean accuracy on CIFAR100, which is about 20\% drop.
In Figure \ref{fig:noisy_mixup}, we show that noisy mixup indeed achieves a thicker boundary than ordinary mixup. We use pure noise to represent OOD, but this simple choice already shows a significant improvement in both OOD and adversarial robustness. This opens the door to devising new mechanisms with the goal of increasing boundary thickness to increase robustness against other forms of image imperfections and/or attacks. In Section \ref{sec:more_experiments_noisy_mixup}, we provide further analysis on noisy mixup.

\subsection{Explaining robustness phenomena using boundary thickness}\label{sec:connect_other_works}

{\bf Robustness to image saturation.} 
We study the connection between boundary thickness and the saturation-based perturbation~\cite{zhang2019interpreting}. 
In \cite{zhang2019interpreting}, the authors show that adversarial training can bias the neural network towards ``shape-oriented'' features and reduce the reliance on ``texture-based'' features. 
One result in \cite{zhang2019interpreting} shows that adversarial training outperforms normal training when the saturation on the images is high. 
In Figure~\ref{fig:saturation}, we show that boundary thickness measured on saturated images in adversarial training is indeed higher than that in normal training.\footnote{We use the online implementation in https://github.com/PKUAI26/AT-CNN.}

\begin{figure}
    \centering
    \begin{subfigure}{0.32\textwidth}
    \includegraphics[width = 1.0\linewidth]{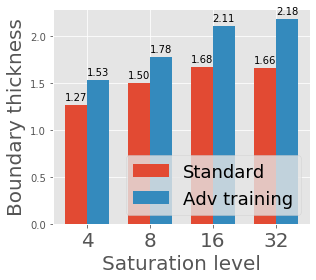}\vspace{-2mm}\subcaption{\label{fig:saturation}}
    \end{subfigure}
    \begin{subfigure}{0.32\textwidth}
    \includegraphics[width = 1.0\linewidth]{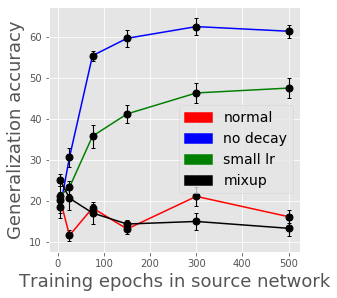}\vspace{-2mm}\subcaption{\label{fig:regularize}}
    \end{subfigure}
    \begin{subfigure}{0.32\textwidth}
    \includegraphics[width = 1.0\linewidth]{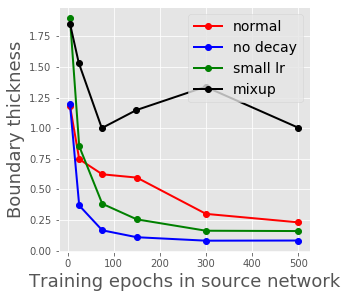}\vspace{-2mm}\subcaption{\label{fig:regularize2}}
    \end{subfigure}
    \caption{
    {\bf Explaining robustness phenomena.}  
    (a) The robustness improvement of adversarial training against saturation-based perturbation (studied in \cite{zhang2019interpreting}) can be explained by a thicker boundary. (b)-(c) 
    Re-implementing the non-robust feature experiment protocol with different training schemes. The two figures show that a thick boundary reduces non-robust features.}
    \label{fig:connect_other_works}
\end{figure}

{\bf A thick boundary reduces non-robust features.} 
We illustrate the connection to \emph{non-robust features}, proposed by~\cite{ilyas2019adversarial} to explain the existence of adversarial examples.
The authors show, perhaps surprisingly, that a neural network trained on data that is completely mislabeled through adversarial attacks can achieve nontrivial generalization accuracy on the clean test data (see Section \ref{sec:mixup_compare} for the experimental protocols of \cite{ilyas2019adversarial} and their specific way of defining generalization accuracy which we use.) 
They attribute this behavior to the existence of non-robust features which are essential for generalization but at the same time are responsible for adversarial vulnerability.

We show that the generalization accuracy defined in this sense decreases if the classifier used to generate adversarial examples has a thicker decision boundary. 
In other words, a thicker boundary removes more non-robust features.
We consider four settings in CIFAR10: 
(1) training without weight decay; 
(2) training with the standard weight decay 5e-4; 
(3) training with the standard weight decay but with a small learning rate $0.003$ (compared to original learning rate $0.1$); and 
(4) training with mixup. 
See Figures \ref{fig:regularize} and \ref{fig:regularize2} for a summary of the results. 
Looking at these two figures together, we see that an increase in the boundary thickness through different training schemes reduces the generalization accuracy, as defined above, and hence the amount of non-robust features retained. 
Note that the natural test accuracy of the four source networks cannot explain the difference in Figure \ref{fig:regularize}, which are 0.943 (``normal''), 0.918 (``no decay''), 0.899 (``small lr''), and 0.938 (``mixup''). 
For instance, training with no weight decay has the highest generalization accuracy defined in the sense above, but its natural accuracy is only 0.918.

%% file: Conclusions.tex
\section{Conclusions}

We introduce boundary thickness, a more robust notion of the size of the decision boundary of a machine learning model, and we provide a range of theoretical and empirical results illustrating its utility.
This includes that a thicker decision boundary reduces overfitting in adversarial training, and that it can improve both adversarial robustness and OOD robustness. 
Thickening the boundary can also reduce boundary tilting and the reliance on ``non-robust features.''
We apply the idea of thick boundary optimization to propose noisy mixup, and we empirically show that using noisy mixup improves robustness. 
We also show that boundary thickness reduces to margin in a special case, but in general it can be more useful than margin.
Finally, we show that the concept of boundary thickness is theoretically justified, by proving that boundary thickness reduces the worst-case boundary tilting and that mixup training achieves the minimax thickness.
Having proved a strong connection between boundary thickness and robustness, we expect that further studies can be conducted with thickness and decision boundaries as their focus.
We also expect that new learning algorithms can be introduced to increase explicitly boundary thickness during training, in addition to the low-complexity but relatively implicit way of noisy mixup.

%% file: Broad_impact.tex
\section{Broader Impact}

The proposed concept of boundary thickness can improve our fundamental understanding of robustness in machines learning and neural networks, and thus it provides new ways to interpret black-box models and improve existing robustness techniques. It can help researchers devise novel training procedures to combat data and model corruption, and it can also provide diagnosis to trained machine learning models and newly proposed regularization or data augmentation techniques.

The proposed work will mostly benefit safety-critical applications, e.g., autonomous driving and cybersecurity, and it may also make AI-based systems more reliable to natural corruptions and imperfections. These benefits are critical because there is usually a gap between the performance of learning-based systems on well-studied datasets and in real-life scenarios.

We will conduct further experimental and theoretical research to understand the limitations of boundary thickness, both as a metric and as a general guideline to design training procedures. We also believe that the machine learning community needs to conduct further research to enhance the fundamental understandings of the structure of decision boundaries and the connection to robustness. For instance, it is useful to design techniques to analyze and visualize decision boundaries during both training (e.g., how the decision boundary evolves) and testing (e.g., how to find defects in the boundary.) We believe this will benefit both the safety and accountability of learning-based systems.

\subsection*{Acknowledgments}
We would like to thank Zhewei Yao, Tianjun Zhang and Dan Hendrycks for their valuable feedback.
Michael W. Mahoney would like to acknowledge the UC Berkeley CLTC, ARO, IARPA (contract W911NF20C0035), NSF, and ONR for providing partial support of this work. 
Kannan Ramchandran would like to acknowledge support from NSF CIF-1703678 and CIF-2002821.
Joseph E. Gonzalez would like to acknowledge supports from NSF CISE Expeditions Award CCF-1730628 and gifts from Amazon Web Services, Ant Group, CapitalOne, Ericsson, Facebook, Futurewei, Google, Intel, Microsoft, Nvidia, Scotiabank, Splunk and VMware.
Our conclusions do not necessarily reflect the position or the policy of our sponsors, and no official endorsement should be~inferred.

%% file: Appendix_mixup_theory.tex
\section{Mixup Increases Thickness}\label{sec:mixup_theory}

In this section, we show that mixup as well as the noisy mixup scheme studied in Section \ref{sec:noisy_mixup} both increase boundary thickness. 

Recall that $x_r$ and $x_s$ in~\eqref{eqn:def_ij} are not necessaraily from the training data. For example, $x_r $ and/or $x_s$ can be the noisy samples used in the noisy mixup (Section \ref{sec:noisy_mixup}). We make the analysis more general here because in different extensions of mixup \cite{zhang2017mixup,lamb2019interpolated,hendrycks2019augmix}, the mixed samples can either come from the training set, from adversarial examples constructed from the training set, or from carefully augmented samples using various forms of image transforms.

We consider binary classification and study the unnormalized empirical version of \eqref{eqn:def_ij} defined as follows:
\begin{equation}\label{eqn:thickness_empirical}
\Theta(f, \alpha, \beta) := \sum_{(x_i, x_j) \text{ s.t. } y_i\neq y_j}\|x_r-x_s\| \int_{t\in [0,1]} \mathbf{I}\{\alpha < g_{01}(x(t)) < \beta\} dt,
\end{equation}
\noindent
where the expectation in~\eqref{eqn:def_ij} is replaced by its empirical counterpart.
We now show that the function which achieves the minimum mixup loss is also the one that achieves minimax thickness for binary classification.

\begin{proposition}[Mixup Increases Boundary Thickness]\label{thm:mixup_minimax}
For binary classification, suppose there exists a function $f_\text{mixup}(x)$ that achieves exactly zero mixup loss, i.e., on all possible pairs of points $((x_r, y_r),  (x_s, y_s))$, $f_\text{mixup}(\lambda x_r+ (1-\lambda)x_s) = \lambda y_r+ (1-\lambda)y_s$ for all $\lambda \in [0,1]$. 
Then, for an arbitrary fixed integer $c>1$, $f_\text{mixup}(x)$ is also a solution to the following minimax problem:
\begin{equation}
     \arg\max_{f} \min_{(\alpha, \beta)} \Theta(f, \alpha, \beta), 
\end{equation}
where the boundary thickness $\Theta$ is defined in Eqn. \eqref{eqn:thickness_empirical}, the maximization is taken over all the 2D functions $f(x) = [f(x)_0, f(x)_1]$ such that $f(x)_0 + f(x)_1 = 1$ for all $x$, and the minimization is taken over all pairs of $\alpha, \beta\in (-1,1)$ such that $\beta-\alpha=1/c$.
\end{proposition}

\begin{proof}
See Section \ref{proof:mixup} for the proof.
\end{proof}

\begin{remark}[Why Mixup is Preferred among Different Thick-boundary Solutions]\label{rem:mixup_reduce_oscillation}
\normalfont
Here, we only prove that mixup provides one solution, instead of the only solution. For example, between two samples $x_r$ and $x_s$ that have different labels, a piece-wise 2D linear mapping that oscillates between $[0,1]$ and $[1,0]$ for more than once can achieve the same thickness as that of a linear mapping. However, a function that exhibits unnecessary oscillations becomes less robust and more sensitive to small input perturbations. Thus, the linear mapping achieved by mixup is preferred. According to \cite{zhang2017mixup}, mixup can also help reduce unnecessary oscillations. 
\end{remark}

\begin{remark}[Zero Loss in Proposition \ref{thm:mixup_minimax}]
\normalfont
Note that the function $f_\text{mixup}(x)$ in the proposition is the one that perfectly fits the mixup augmented dataset. In other words, the theorem above needs $f_\text{mixup}(x)$ to have ``infinite capacity,'' in some sense, to match perfectly the response on line segments that connect pairs of points $(x_r,x_s)$. If such $f_\text{mixup}(x)$ does not exist, it is unclear if an approximate solution achieves minimax thickness, and it is also unclear if minimizing the cross-entropy based mixup loss is exactly equivalent to minimizing the minimax boundary thickness for the same loss value. Nonetheless, our experiments show that mixup consistently achieves thicker decision boundaries than ordinary training (see Figure \ref{fig:thickness_adv_direction}).
\end{remark}

%% file: Appendix_proofs.tex
\section{Proofs}

\subsection{Proof of Proposition \ref{prop:linear}}\label{proof:linear}

Choose $(x_r, x_s)$ so that $x_s - x_r = cw$. The thickness of $f$ defined in \eqref{eqn:def_ij} becomes
\begin{equation}\label{eqn:thickness_linear}
\Theta(f, \alpha, \beta, x_s, x_r) = c\|w\| \mathbb{E}_p\left[\int_0^1 \mathbf{I} \{\alpha < g_{01}(x(t)) <\beta\}dt\right].
\end{equation} 
Define a substitute variable $u$ as:
\begin{equation}
    u = tw^\top x_r + (1-t)w^\top x_s + b.
\end{equation}
Then, 
\begin{equation}
    du = (w^\top x_r - w^\top x_s)dt = w^\top (-cw) dt = -c\|w\|^2 dt.
\end{equation}
Further, 
\begin{equation}
    f(t x_r + (1-t) x_s)_0 - f(t x_r + (1-t) x_s)_1 = 2\sigma(t w^\top x_r + (1-t) w^\top x_s + b) - 1 = 2\sigma(u)-1 = \tilde{g}(u).
\end{equation}
Thus,
\begin{equation}
\begin{split}
    \Theta(f, \alpha, \beta, p) \overset{(a)}{=} &c\|w\| \mathbb{E}\left[\int_0^1 \mathbf{I} \{\alpha < f(t x_r + (1-t) x_s)_0 - f(t x_r + (1-t) x_s)_1 <\beta\}dt\right] \\
    \overset{(b)}{=}& \mathbb{E}\left[\int_{w^\top x_s + b}^{w^\top x_r + b} \mathbf{I}\{\alpha < \tilde{g}(u) < \beta\} \left(-\frac{1}{\|w\|}\right)du\right]\\
    \overset{(c)}{=} &\frac{1}{\|w\|} \mathbb{E}\left[\int_{w^\top x_r + b}^ {w^\top x_s + b} \mathbf{I}\{\alpha < \tilde{g}(u) < \beta\} du\right],
\end{split}
\end{equation}
where $(a)$ holds because $g_{01}(x(t)) = f(x(t))_0 - f(x(t))_1 = f(t x_r + (1-t) x_s)_0 - f(t x_r + (1-t) x_s)_1$, $(b)$ is from substituting $u = tw^\top x_r + (1-t)w^\top x_s + b$ and $du = -c\|w\|^2 dt$, and $(c)$ is from switching the upper and lower limit of the integral to get rid of the negative sign. Recall that $\tilde{g}(u)$ is a monotonically increasing function in $u$. Thus, 
\begin{equation}
\begin{split}
    \Theta(f, \alpha, \beta, p)  = & \frac{1}{\|w\|} \int_{w^\top x_r + b}^{w^\top x_s + b} \mathbf{I}\{\tilde{g}^{-1}(\alpha) < u < \tilde{g}^{-1}(\beta)\} du \\
= & \frac{1}{\|w\|} \mathbb{E}\left[\min(\tilde{g}^{-1}(\beta),w^\top x_s + b) - \max(\tilde{g}^{-1}(\alpha), w^\top x_r + b)\right].
\end{split}
\end{equation}
Further, if $[\alpha,\beta]$ is contained in $[g_{01}(x_r), g_{01}(x_s)]$, we have $\tilde{g}^{-1}(\beta) < \tilde{g}^{-1}(g_{01}(x_s)) = \tilde{g}^{-1}(\tilde{g}(w^\top x_s + b))
= w^\top x_s + b$, and similarly,  $\tilde{g}^{-1}(\alpha) > w^\top x_r + b$, and thus
\begin{equation}
    \Theta(f, \alpha, \beta) = (\tilde{g}^{-1}(\beta) - \tilde{g}^{-1}(\alpha))/\|w\|.
\end{equation}

\subsection{Proof of Proposition \ref{prop:margin}}\label{proof:margin}

The conclusion holds if $\|x_r-x_s\| \int_0^1 \mathbf{I} \{\alpha < g_{ij}(x(t)) <\beta\}dt$ equals the $\ell_2$ distance from $x_r$ to its projection $\text{Proj}(x,j)$ for $\alpha=0$ and $\beta=1$. Note that when $x = x_s$, $g_{ij}(x) = 0$, because $x_s = \text{Proj}(x_r,j)$. From the definition of projection, i.e., $\text{Proj}(x, j) = \arg\min_{x' \in S(i_x,j)} \| x' - x\|$, we have that for all points $x$ on the segment from $x_r$ to $x_s$, $x_s$ is only point with $g_{ij}(x) = 0$. Otherwise, $x_s$ is not the projection. Therefore, all points $x$ on the segment satisfy $g_{ij}(x) > 0 =\alpha$. Since $f$ is the output after the softmax layer, $g_{ij}(x) = f(x)_{i} - f(x)_{j} < 1 = \beta$. Thus, the indicator function on the left-hand-side of \eqref{eqn:thickness_reduces_to_margin} takes value 1 always, and the integration reduces to calculating the distance from $x_r$ to $x_s$.

\subsection{Proof of Proposition \ref{prop:tilting}}\label{proof:tilting}
We rewrite the definition of $T(u)$ as
\begin{equation}\label{eqn:CSu}
    T(u) := \min_{v\text{ s.t. }\|v\| = u \text{ and } y_i v^\top x_i\ge 1, \forall i } \text{Cosine Similarity}(v, w^*),
\end{equation}
where 
\begin{equation}\label{eqn:tilting_linear}
    \text{Cosine Similarity}(v, w^*) := |v^\top w^*|/(\|v\| \cdot \|w^*\|).
\end{equation}
To prove $T(u)$ is a non-increasing function in $u$, we consider arbitrary $u_1, u_2$ so that $u_1>u_2\ge\|w^*\|$, and we prove $T(u_1)\le T(u_2)$. 

First, consider $T(u_2)$. Denote by $w_2$ the linear classifier that achieves the minimum value in the RHS of \eqref{eqn:CSu} when $u=u_2$. From definition, $u_2 = \|w_2\|$. Now, if we increase the norm of $w_2$ to obtain a new classifier $w_1 = \frac{u_1}{u_2} w_2$, it still satisfies the constraint $y_i w_1 x_i\ge 1, \forall i$ because 
\begin{equation}
    y_i w_1 x_i = \frac{u_1}{u_2} y_i w_2 x_i \ge \frac{u_1}{u_2} >1.
\end{equation}
Thus, $w_1= \frac{u_1}{u_2} w_2$ satisfies the constraints in \eqref{eqn:CSu} for $u=\|w_1\| = u_1$, and being a linear scaling of $w_2$, it has the same cosine similarity score with $w^\star$~\eqref{eqn:tilting_linear}, which means the worst-case tilting $T(u_1)$ should be smaller or equal to the tilting of $w_1$. 

\subsection{Proof of Proposition \ref{thm:mixup_minimax}}\label{proof:mixup}

We can rewrite \eqref{eqn:thickness_empirical} using
\begin{equation}
\Theta(f, \alpha, \beta) := \sum_{(x_i, x_j) \text{ s.t. } y_i\neq y_j}\Theta_\text{1D}(f, \alpha, \beta, x_r, x_s),
\end{equation}
where $\Theta_\text{1D}(f, \alpha, \beta, x_r, x_s)$ denotes the thickness measured on a single segment, i.e.,
\begin{equation}\label{eqn:thickness_1D}
    \Theta_\text{1D}(f, \alpha, \beta, x_r, x_s) := \|x_r-x_s\| \int_{t\in [0,1]} \mathbf{I}\{\alpha < g_{01}(x(t)) < \beta\} dt,
\end{equation}
where recall that $g_{01}(x) = f(x)_0 - f(x)_1$ and $x(t) = tx_r+(1-t)x_s$.

Since the proposition is stated for the sum on all pairs of data, we can focus on the proof of an arbitrary pair of data $(x_r, x_s)$ such that $y_r\neq y_s$. 

Consider any 2D decision function $f(x) = [f(x)_0, f(x)_1]$ such that $f(x)_0+f(x)_1 = 1$ (i.e., $f(x)$ is a probability mass function). In the following, we consider the restriction of $f(x)$ on a segment $(x_r, x_s)$, which we denote as $f_{(x_r, x_s)}(x)$. Then, the proof relies on the following lemma, which states that the linear interpolation scheme in mixup training does maximize the boundary thickness on the segment.

\begin{lemma}\label{lmm:straight_line}
For any arbitrary fixed integer $c>0$, the linear function 
\begin{equation}\label{eqn:lin}
f_\text{lin}(t x_r+ (1-t)x_s)= [t, 1-t], t\in [0,1],
\end{equation}
defined for a given segment ($x_r$, $x_s$) optimizes $\Theta_\text{1D}(f,\alpha,\beta,x_r,x_s)$ in \eqref{eqn:thickness_1D} in the following minimax sense,
\begin{equation}\label{eqn:minimax}
    f_\text{lin}(x) = \arg\max_{f_{(x_r, x_s)}(x)} \min_{(\alpha, \beta)} \Theta_\text{1D}(f_{(x_r, x_s)}(x), \alpha, \beta, x_r, x_s),
\end{equation}
where the maximization is over all the 2D functions $f_{(x_r, x_s)}(x) = [f_{(x_r, x_s)}(x)_0,f_{(x_r, x_s)}(x)_1]$ such that the domain is restricted to the segment ($x_r$, $x_s$) and such that $f_{(x_r, x_s)}(x)_0, f_{(x_r, x_s)}(x)_1\in [0,1]$ and $f_{(x_r, x_s)}(x)_0+f_{(x_r, x_s)}(x)_1 = 1$ for all $x$ on the segment, and the minimization is taken over all pairs of $\alpha, \beta\in (-1,1)$ such that $\beta-\alpha=1/c$.
\end{lemma}
\begin{proof}See Section \ref{sec:1D_proof} for the proof.\end{proof}

Now, Proposition \ref{thm:mixup_minimax} follows directly from Lemma \ref{lmm:straight_line}.

\subsection{Proof of Lemma \ref{lmm:straight_line}}\label{sec:1D_proof}
In this proof, we simplify the notation and use $f(x)$ to denote $f_{(x_r, x_s)}(x)$ which represents $f(x)$ restricted to the segment $(x_r, x_s)$. This simple notation does not cause any confusion because we restrict to the segment $(x_r, x_s)$ in this proof.

We can simplify the proof by viewing the optimization over functions $f(x)$ on the fixed segment $(x_r, x_s)$ as optimizing over the functions $h(t) = f(x(t))_1-f(x(t))_0$ on $t\in [0,1]$, where $x(t) = tx_r +(1-t)x_s$. 

Thus, we only need to find the function $f(x)$, when viewed as a one-dimensional function $h(t) = f(x(t))_1 - f(x(t))_0 = f(tx_r + (1-t)x_s)_1 - f(tx_r + (1-t)x_s)_0$, that solves the minimax problem \eqref{eqn:minimax} for the thickness defined as:
\begin{equation}\label{eqn:thickness_01}
\begin{split}
    \Theta_\text{1D}(f(t)) 
    =& \|x_r-x_s\| \int_{t\in [0,1]} \mathbf{I}\{\alpha < h(t) < \beta\} dt\\
    =& \|x_r-x_s\| |h^{-1}((\alpha, \beta))| ,
\end{split}
\end{equation}
where $h^{-1}$ is the inverse function of $h(t)$. Note that $h^{-1}((\alpha, \beta)) \subset [0,1]$. To prove the result, we only need to prove that the linear function $h_\text{lin}(t) = 2t-1$, which is obtained from $h(t) = f(t x_r+ (1-t)x_s)_1 - f(t x_r+ (1-t)x_s)_0$ for $f_\text{lin}(t x_r+ (1-t)x_s)= [t, 1-t]$ defined in \eqref{eqn:lin}, solves the minimax problem 
\begin{equation}\label{eqn:1D_minimax}
\begin{split}
    & \arg\max_{h(t)} \min_{(\alpha, \beta)} \Theta_\text{1D}(h(t))
     = \arg\max_{h(t)} \min_{(\alpha, \beta)} |h^{-1}((\alpha, \beta))|,
\end{split}
\end{equation}
where the maximization is taken over all $h(t)$, and the minimization is taken over all pairs of $\alpha, \beta\in (-1,1)$ such that $\beta-\alpha=1/c$, for a fixed integer $c>1$. 

Now we prove a stronger statement.

{\bf Stronger statement:} 
\begin{equation}
    \min_{(\alpha, \beta)} |h^{-1}((\alpha, \beta))| \le \frac{\beta - \alpha}{2},
\end{equation}
when the minimization is taken over all $\alpha, \beta$ such that $\alpha - \beta = 1/c$, and for any measurable function $h(t): [0,1]\to [-1,1]$. 

If we can prove this statement, then, since $h_{\text{lin}}(t)=2t-1$ always achieves $h^{-1}((\alpha, \beta)) = \frac{\beta+1}{2} - \frac{\alpha+1}{2} = \frac{\beta-\alpha}{2}$, it is indeed the minimax solution of \eqref{eqn:1D_minimax}.

We prove the stronger statement above by contradiction. Suppose that the statement is not true, i.e., for any $\alpha$ and $\beta$ such that $\beta - \alpha = 1/c$, we always have 
\begin{equation}\label{eqn:contradiction}
 |h^{-1}((\alpha, \beta))| >  \frac{\beta-\alpha}{2} = \frac{1}{2c}.
\end{equation}
Then, the pre-image of $[-1,1]$ satisfies
\begin{equation}
\begin{split}
1\ge &  |h^{-1}([-1,1])|\\
= &\sum_{i=-c+1}^c \left|h^{-1}\left((\frac{i-1}{c}, \frac{i}{c})\right)\right| \\
\overset{(a)}{>}& 2c\cdot \frac{1}{2c} = 1,
\end{split}
\end{equation}
where the last inequality holds because of the inequality \eqref{eqn:contradiction}. This is clearly a contradiction, which means that the stronger statement is true.

%% file: Appendix_chess_board.tex
\section{A Chessboard Toy Example}\label{sec:chess_board}

\begin{figure}[ht]
\centering
\begin{tabular}{ccccc}
\hline
\hline
 & {\bf 2D chessboard data} & & {\bf 2D classifier} & \\
\hline
 & \includegraphics[width=.125\columnwidth,height=.125\columnwidth,keepaspectratio,valign=m,margin=.0cm .05cm]{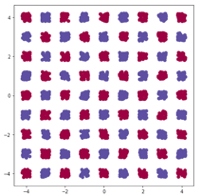} & &
\includegraphics[width=.125\columnwidth,height=.125\columnwidth,keepaspectratio,valign=m,margin=.0cm .05cm]{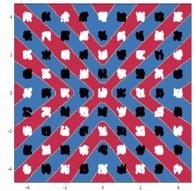} & \\
\hline
\hline
{\bf 3D Robust classifier} & & & {\bf 3D Non-robust classifier} & \\
Front view & Top-down view & & Front view & Top-down view \\
\hline
\includegraphics[width=.125\columnwidth,height=.125\columnwidth,keepaspectratio,valign=m,margin=.0cm .0cm]{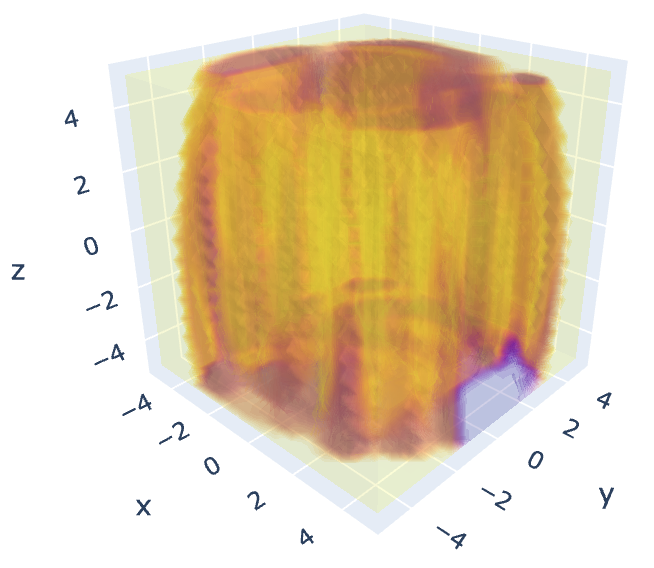}
&
\includegraphics[width=.125\columnwidth,height=.125\columnwidth,keepaspectratio,valign=m,margin=.0cm .05cm]{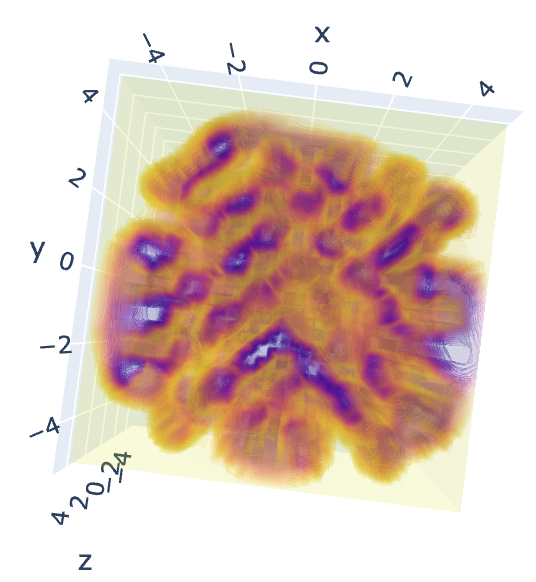} 
& 
\includegraphics[width=.025\columnwidth,height=.125\columnwidth,keepaspectratio,valign=m,margin=.0cm .05cm]{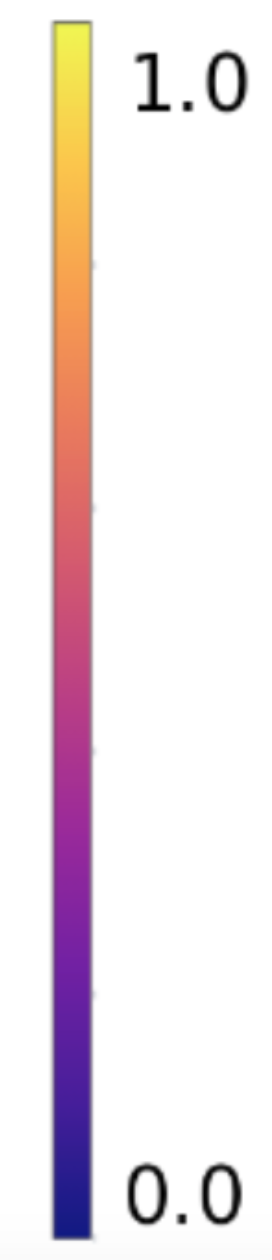} 
&
\includegraphics[width=.125\columnwidth,height=.125\columnwidth,keepaspectratio,valign=m,margin=.0cm .0cm]{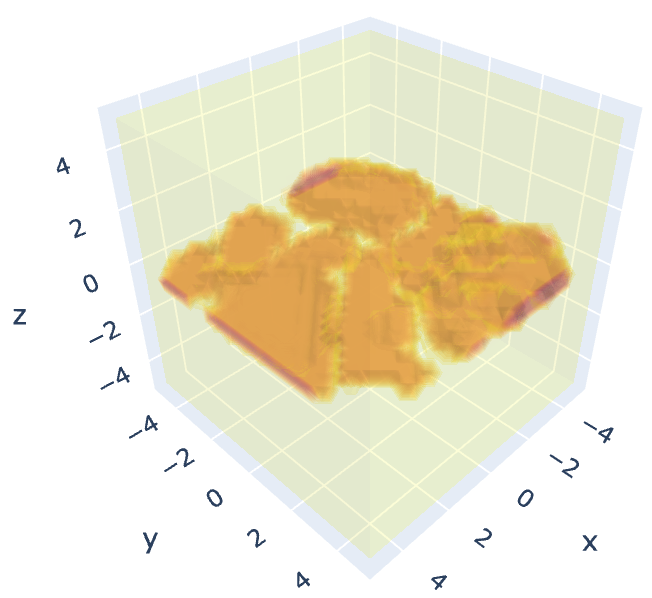}
&
\includegraphics[width=.125\columnwidth,height=.125\columnwidth,keepaspectratio,valign=m,margin=.0cm .05cm]{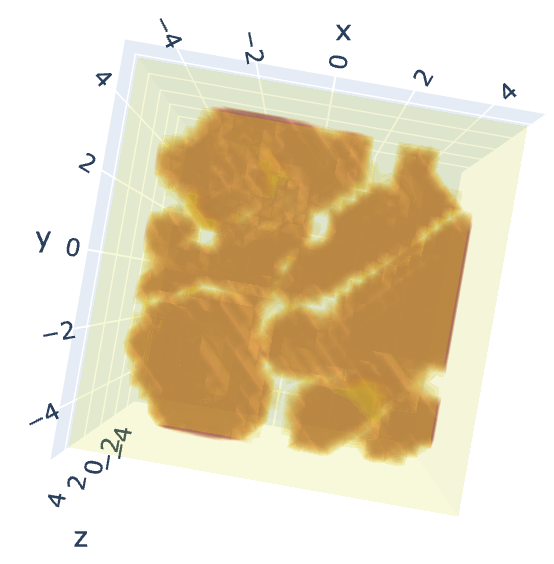} 
%& 
%\includegraphics[width=.025\columnwidth,height=.125\columnwidth,keepaspectratio,valign=m,margin=.0cm .05cm]{figs/color_bar.png}
\\
\hline
\hline
\multicolumn{5}{c}{\bf Interpolation between robust and non-robust classifier}\\
\hline
\multicolumn{5}{c}{\includegraphics[width=.99\textwidth]{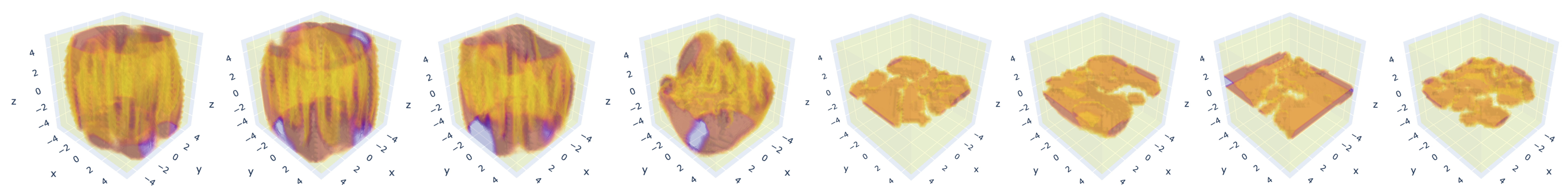}}\\
\hline
\hline
\end{tabular}
\medskip
\caption{{\bf Chessboard toy example.} The 3D visualizations above use the color map that ranges from yellow (value 1) to purple (value 0) to illustrate the predicted probability $\Pr(y=1|x)\in [0,1]$ in binary classification. Each 3D figure draws 17 level sets of different colors from 0 to 1.
\\
{\bf First row: } The 2D chessboard dataset with two classes and a 2D classifier that learns the correct pattern.\\
{\bf Second row:} 3D visualization of decision boundaries of two different classifiers. {\bf (left) }A classifier that uses robust $x$ and $y$ directions to classify, which preserves the complex chessboard pattern (see the top-down view which contains a chessboard pattern.)
{\bf (right) }A classifier that uses the non-robust direction $z$ to classify.
When the separable space on the non-robust direction is large enough, the thin boundary squeezes in and generates a simple but non-robust function.\\
{\bf Third row: }
Visulization of the decision boundary as we interpolate between the robust and non-robust classifiers. There is a sharp transition from the fourth to the fifth figure.
}\label{fig:chess_board}
\end{figure}

In this section, we use a chessboard example in low dimensions to show that nonrobustness arises from thin boundaries. Being a toy setting, this is limited in the generality, but it can visually demonstrate our main~message. 

In Figure \ref{fig:chess_board}, the two subfigures shown on the first row represent the chessboard dataset and a robust 2D function that correctly classifies the data. Then, we project the 2D chessboard data to a 100-dimensional space by padding noise to each 2D point. In this case, the neural network can still learn the chessboard pattern and preserve the robust decision boundary (see the 3D top-down view on the left part of the second row which contains the chessboard pattern).

However, if we randomly perturb each square of samples in the chessboard in the 3rd dimension (the $z$ axis) to change the space between these squares, such that the boundary has enough space on the $z$-axis to partition the two opposite classes, the boundary changes to a non-robust one instead (see the right part on the second row of Figure \ref{fig:chess_board}). The shift value on the $z$ axis is $0.05$ which is much smaller than the distance between two adjacent squares, which is 0.6. The data are not linearly separable on the $z$-axis because each square on the chessboard is randomly shifted up or down independently of other squares. 

A more interesting result can be obtained by varying the shift on the $z$ axis from 0.01 to 0.08. See the third row of Figure \ref{fig:chess_board}. The network undergoes a sharp transition from using robust decision boundaries to using non-robust ones. This is consistent with the main message shown in Figure \ref{fig:max-margin}, i.e., that neural networks tend to generate thin and non-robust boundaries to fit in the narrow space between opposite classes on the non-robust direction, while a thick boundary mitigates this effect. On the third row, from left to right, the expanse of the data on $z$-axis increases, allowing the network to use only the $z$-axis to classify. %

{\bf Details of 3D figure generation:} For the visualization in Figure \ref{fig:chess_board}, each 3D figure is generated by plotting 17 consecutive level sets of neural network prediction values (after the softmax layer) from 0 to 1. The prediction on each level set is the same, and each level set is represented by a colored continuous surface in the 3D figure. The yellow end of the color bar represents a function value of 1, and the purple end represents 0. The visualization is displayed in a 3D orthogonal subspace of the whole input space. The three axes are the first three vectors in the natural basis. They represent the $x$ and $y$ directions that contain the chessboard pattern, and the $z$ axis that contains the direction of shift values. 

\textbf{Details of the chessboard toy example: }The chessboard data contains 2 classes of 2D points arranged in $9\times 9 = 81$ squares. Each square contains 100 randomly generated 2D points uniformly distributed in the square. The length of each square is 0.4, and the separation between two adjacent squares is 0.6. The shift direction (up or down) and value on the $z$-axis are the same for all 2D points in a single square, and the shift value is much smaller than 1 (which is the distance between the centers of two squares). See the third row on Figure \ref{fig:chess_board} for different shift values ranging from 0.01 to 0.08. The shift value is, however, independent across different squares, i.e., these squares cannot be easily separated by a linear classifier using information on the $z$-axis only. The classifier is a neural network with 9 fully-connected layers and a residual link on each layer. The training has 100 epochs, an initial learning rate of 0.003, batch size 128, weight decay 5e-4, and momentum 0.9.

%% file: Appendix_boundary_tilting.tex
\section{A Thick Boundary Mitigates Boundary Tilting}\label{sec:thick_boundary_reduces_tilting}

\begin{figure}
    \centering
    \begin{subfigure}{0.32\textwidth}
    \vspace{-5mm}
        \includegraphics[width=.98\linewidth]{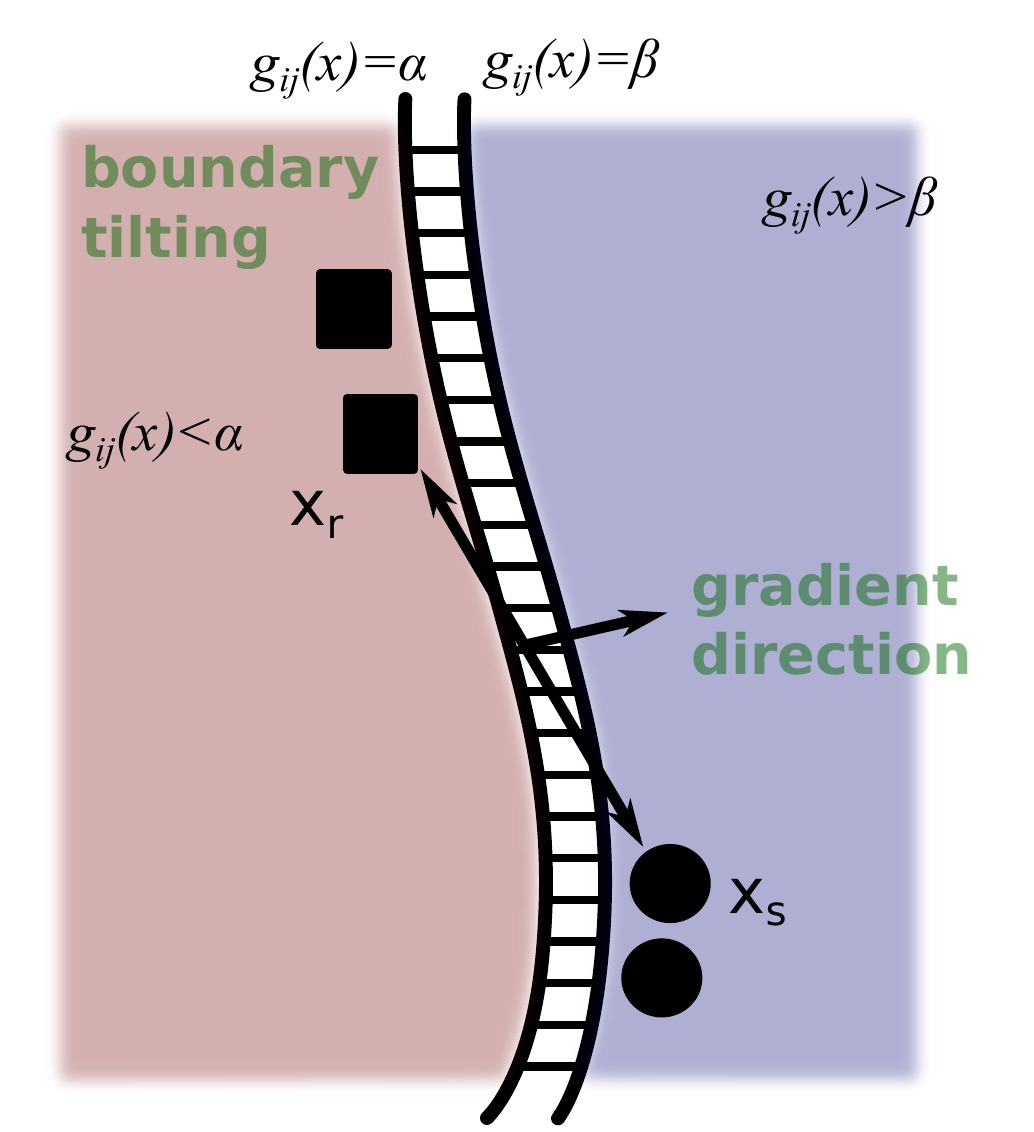}
        \vspace{3mm}
        \subcaption{\label{fig:tilt_explain}}
    \end{subfigure}
    \begin{subfigure}{0.45\textwidth}
        \includegraphics[width=.98\linewidth]{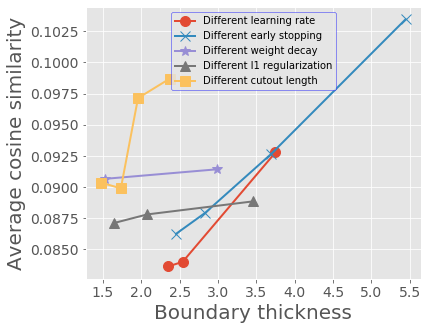}
        \subcaption{\label{fig:tilt_exp}}
    \end{subfigure}
    \caption{{\bf Thickness and boundary tilting.} (a) The cosine similarity between the gradient direction and $x_s-x_r$ generalizes the measurement of ``boundary tilting'' to nonlinear functions. (b) Boundary tilting can be mitigated by using a thick decision boundary.}
    \label{fig:tilt}
\end{figure}

In this section, we generalize the observation of Proposition \ref{prop:tilting} to nonlinear classifiers. Recall that in Proposition \ref{prop:tilting}, we use Cosine Similarity $(w, w^*) = |w^\top w^*|/(\|w\| \cdot \|w^*\|)$ between the classifier $w$ and the max-margin classifier $w^*$ to measure boundary tilting. To measure boundary tilting in the nonlinear case, we use $x_1-x_2$ of random sample pairs ($x_1$, $x_2$) from the training set to replace the normal direction of the max-margin solution $w^*$, and use $\nabla g_{ij}(x) =\nabla (f(x)_i - f(x)_j)$ to replace the normal direction of a linear classifier $w$, where $i, j$ are the predicted labels of $x_1$ and $x_2$, respectively, and $x$ is a random point on the line segment $(x_1 , x_2)$. 
 Then, the cosine similarity generalizes to
\begin{equation}\label{eqn:cos}
    \text{Average Cosine Similarity} = \mathbb{E}_{(x_1,x_2)\sim \text{training distribution s.t. }i\neq j }\left[\frac{|\langle x_1 - x_2, \nabla_x g_{ij}(x) \rangle |}{\|x_1 - x_2\|\|\nabla_x g_{ij}(x)\|}\right].
\end{equation}

In Figure \ref{fig:tilt_explain}, we show the intuition underlying the use of~\eqref{eqn:cos}. The smaller the cosine similarity is, the more severe the impact of boundary tilting becomes. 

We also measure boundary tilting in various settings of adversarial training, and we choose the same set of hyper-parameters that are used to generate Figure \ref{fig:adv_thickness_all}. See the results in Figure \ref{fig:tilt_exp}. When measuring cosine similarity, we average the results over 6400 training sample pairs. From the results shown in Figure \ref{fig:tilt_exp}, a thick boundary mitigates boundary tilting by increasing the cosine similarity.

%% file: Appendix_non_adv.tex
\section{Additional Experiments on Non-adversarial Training}\label{sec:non_adv_details}

In this section, we provide more details and additional experiments extending the results of Section \ref{sec:regularization} on non-adversarially trained neural networks. We demonstrate that a thick boundary improves OOD robustness \textcolor{black}{when the thickness is measured using different choices of hyper-parameters}. We also show that the same conclusion holds on two other datasets, namely CIFAR100 and SVHN, in addition to CIFAR10 used in the main paper.

\subsection{Details of measuring boundary thickness}\label{sec:details_measure_thickness}

Boundary thickness is calculated by integrating on the segments that connect a sample with its corresponding adversarial sample. We find the adversarial sample by using an $\ell_2$ PGD attack of size 1.0, step size 0.2, and number of attack steps 20. We measure both thickness and margin on the normalized images in CIFAR10, which introduces a multiplicity factor of approximately 5 when using the standard deviations $(0.2023, 0.1994, 0.2010)$, respectively, for the RGB channels compared to measuring thickness on unnormalized~images. 

To compute the integral in \eqref{eqn:def_ij}, we connect the segment from $x_r$ to $x_s$ and evaluate the neural network response on 128 evenly spaced points on the segment. Then, we compute the cumulative $\ell_2$ distance of the parts on this segment for which the prediction value is between $(\alpha, \beta)$, which measures the distance between two level sets $g_{ij}(x) = \alpha$ and $g_{ij}(x) = \beta$ on this segment (see equation \eqref{eqn:def_ij}). Finally, we report the average thickness obtained by repeated runs on 320 segments, i.e., 320 random samples with their adversarial~examples.

\subsection{Comparing different measuring methods: tradeoff between complexity and accuracy}\label{sec:compare_different_thickness}

In this section, we discuss the choice of distribution $p$ when selecting segments $(x_r, x_s)$ to measure thickness. Recall that in the main paper, we choose $x_s$ as an adversarial example of $x_r$. Another way, which is computationally cheaper, is to measure thickness on the segment directly between pairs of samples in the training dataset, i.e., sample $x_r$ randomly from the training data, and sample $x_s$ as a random data point with a different label.

Although computationally cheaper, this way of measuring boundary thickness is more prone to the ``boundary-tilting'' effect, because the connection between a pair of samples is not guaranteed to be orthogonal to the decision boundary. \textcolor{black}{Thus, the boundary tilting effect can inflate the value of boundary thickness. This effect only happens when we measure thickness on pairs of samples instead of measuring it in the adversarial direction, which we have shown to be able to mitigate boundary tilting when the thickness is large (see Section~\ref{sec:thick_boundary_reduces_tilting}).}

In Figure \ref{fig:thickness_sample_pairs}, we show how this method affects the measurement of thickness. The thickness is measured for the same set of models and training procedures as those shown in Figure \ref{fig:thickness_adv_direction}, but on random segments that connect pairs of samples. We use $\alpha = 0$ and $\beta=0.75$ to match Figure \ref{fig:thickness_adv_direction}. In Figure \ref{fig:thickness_sample_pairs}, although the trend remains the same (i.e., mixup$>$normal$>$training without weight decay), all the measurement values of boundary thickness become much bigger than that of Figure \ref{fig:thickness_adv_direction}, indicating boundary tilting in all the measured networks.

\begin{figure}
    \centering
     \includegraphics[width=\linewidth]{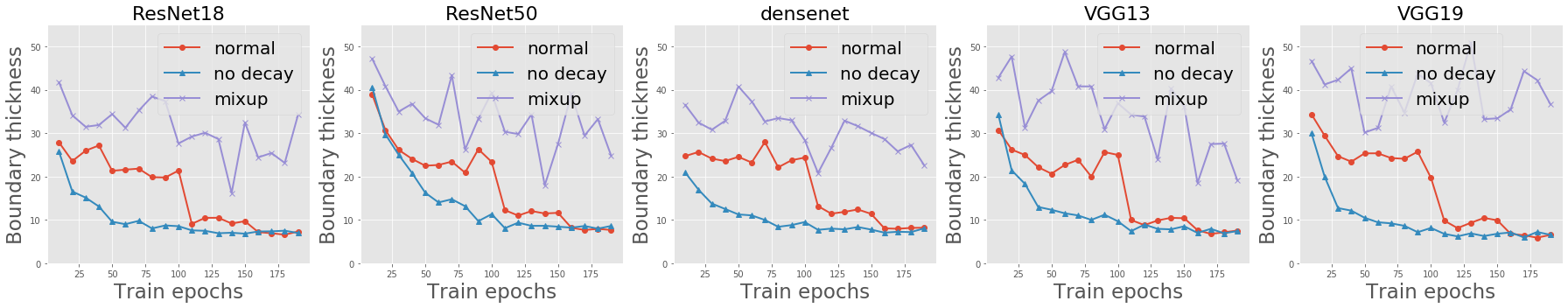}
    \caption{{\bf Thickness on random sample pairs.} Measuring the boundary thickness in the same experimental setting as Figure \ref{fig:thickness_adv_direction}, but on pairs of random samples. The trend that mixup $>$ normal training $>$ training without weight decay remains the same.}
    \label{fig:thickness_sample_pairs}
\end{figure}

\begin{remark}[An Oscillating 1D Example Motivates the Adversarial Direction]
\normalfont
Obviously, the distribution $p$ in Definition \ref{def:boundary_thickness} is vital in dictating robustness.
Similar to Remark \ref{rem:mixup_reduce_oscillation}, one can consider an example of 2D piece-wise linear mapping $f(x) = [f(x)_0, f(x)_1]$ on a segment ($x_r$, $x_s$) that oscillates between the response $[0,1]$ and $[1,0]$. If one measures the thickness on this particular segment, the measured thickness remains the same if the number of oscillations increases in the piece-wise linear mapping, but the robustness reduces with more oscillations. Thus, the example motivates the measurement on the direction of an adversarial attack, because an adversarial attack tends to find the closest ``peak'' or ``valley'' and can thus faithfully recover the correct value of boundary thickness unaffected by the oscillation.
\end{remark}

\subsection{Ablation study}\label{sec:ablation_non_adv}

In this section, we provide an extensive ablation study on the different choices of hyper-parameters used in the experiments. \textcolor{black}{We show that our main conclusion about the positive correlation between robustness and thickness remains the same for a wide range of hyper-parameters obviating the need to fine-tune these.} We study the adversarial attack direction used to measure thickness, the parameters $\alpha$ and $\beta$, as well as reproducing the results on two other datasets, namely CIFAR100 and SVHN, in addition to CIFAR10.

\subsubsection{Different choices of adversarial attack in measuring boundary thickness}\label{sec:ablation_adv_attack_parameters}

\begin{figure}
    \centering
    
    \begin{subfigure}{0.98\textwidth}
    \includegraphics[width=\linewidth]{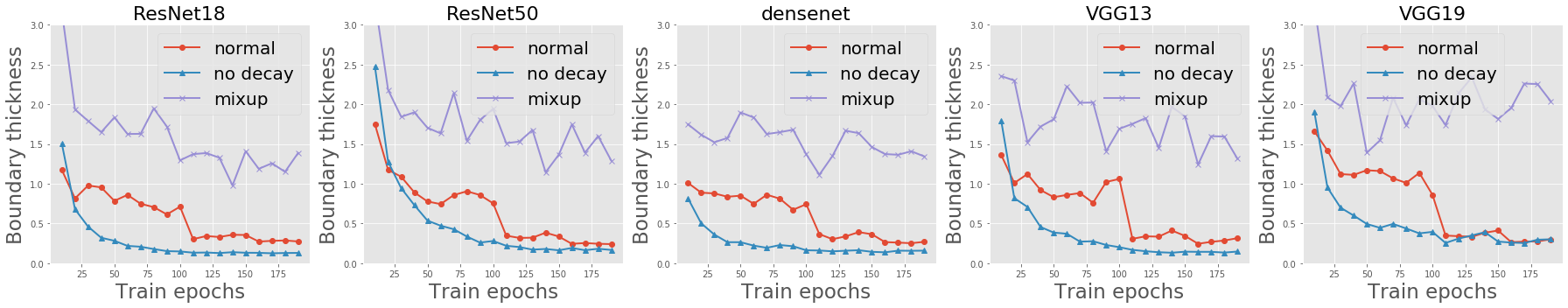}
    \subcaption{Results on CIFAR10 with a large attack $\epsilon$ = 2.0}
    \end{subfigure}
    
    \begin{subfigure}{0.98\textwidth}
    \includegraphics[width=\linewidth]{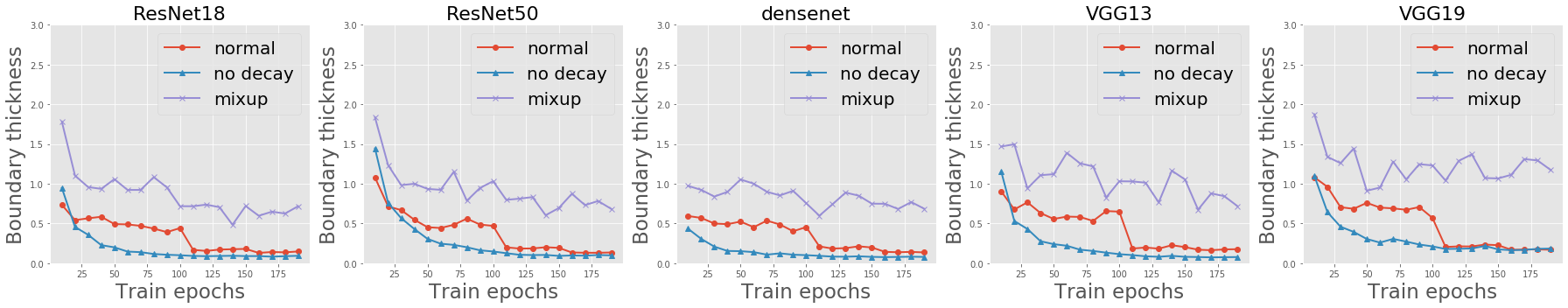}
    \subcaption{Results on CIFAR10 with a small attack $\epsilon$ = 0.6}
    \end{subfigure}
    
    \caption{{\bf Ablation study on different attack sizes.} Re-implementing the measurements in Figure \ref{fig:thickness_adv_direction} using a larger or a smaller adversarial attack.}
    \label{fig:different_attack}
\end{figure}

To measure boundary thickness on the adversarial direction, we have to specify a way to implement the adversarial attack. To generate Figure \ref{fig:thickness_adv_direction}, we used $\ell_2$ attack with attack range $\epsilon =$1.0, step size 0.2, and number of attack steps 20. We show that the results and more importantly our conclusions do not change by perturbing $\epsilon$ a little. See Figure \ref{fig:different_attack} and compare it with the corresponding results presented in Figure \ref{fig:thickness_adv_direction}. We see that the change in the size of the adversarial attack does not alter the trend. However, the measured thickness value does shrink if the $\epsilon$ becomes too small, which is expected.

\subsubsection{Different choices of $\alpha$ and $\beta$ in measuring boundary thickness}

\begin{figure}[ht!]
    \centering
    
    \begin{subfigure}{0.98\textwidth}
    \includegraphics[width=\linewidth]{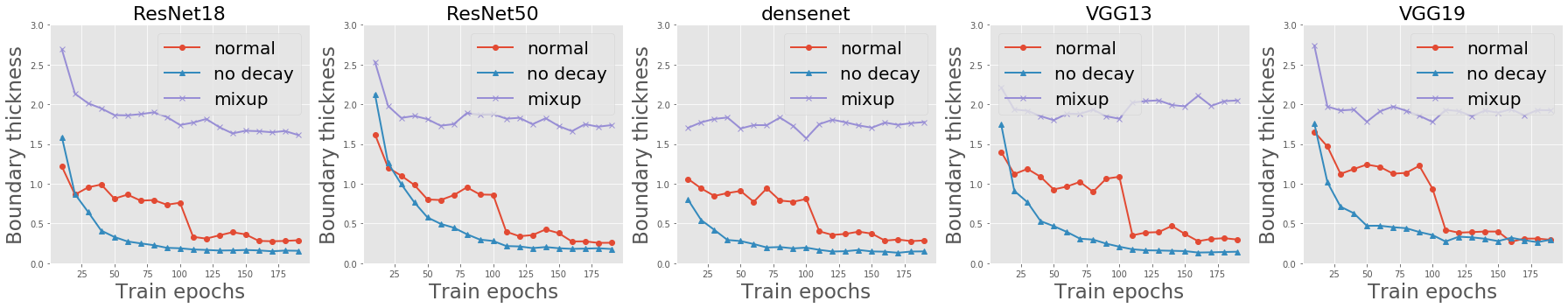}
    \subcaption{Results on CIFAR10 for $\alpha =0$ and $\beta =0.9$}
    \end{subfigure}
    
    \begin{subfigure}{0.98\textwidth}
    \includegraphics[width=\linewidth]{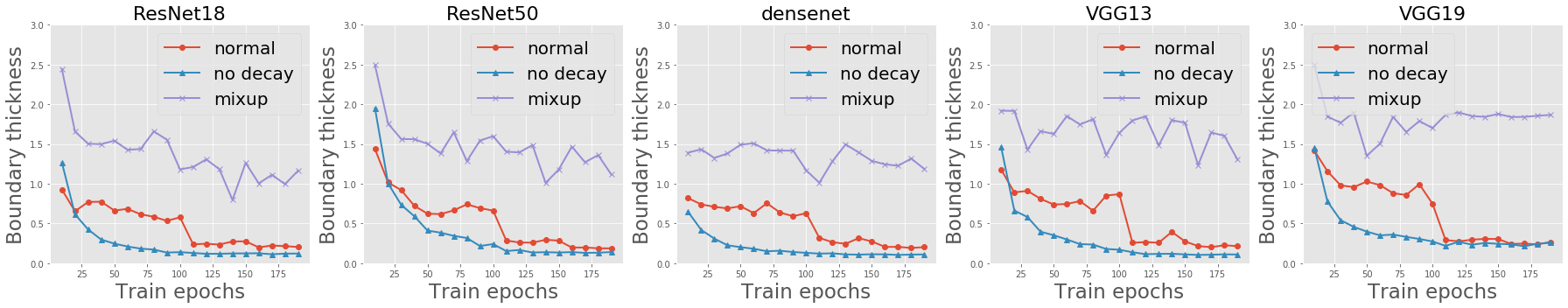}
    \subcaption{Results on CIFAR10 for $\alpha =0$ and $\beta =0.8$}
    \end{subfigure}
    
    \begin{subfigure}{0.98\textwidth}
    \includegraphics[width=\linewidth]{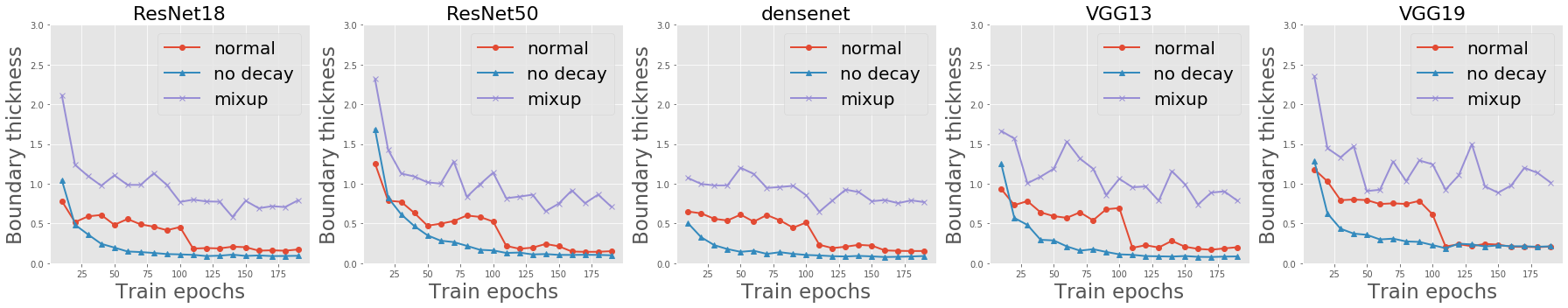}
    \subcaption{Results on CIFAR10 for $\alpha =0$ and $\beta =0.7$}
    \end{subfigure}
    
    \begin{subfigure}{0.98\textwidth}
    \includegraphics[width=\linewidth]{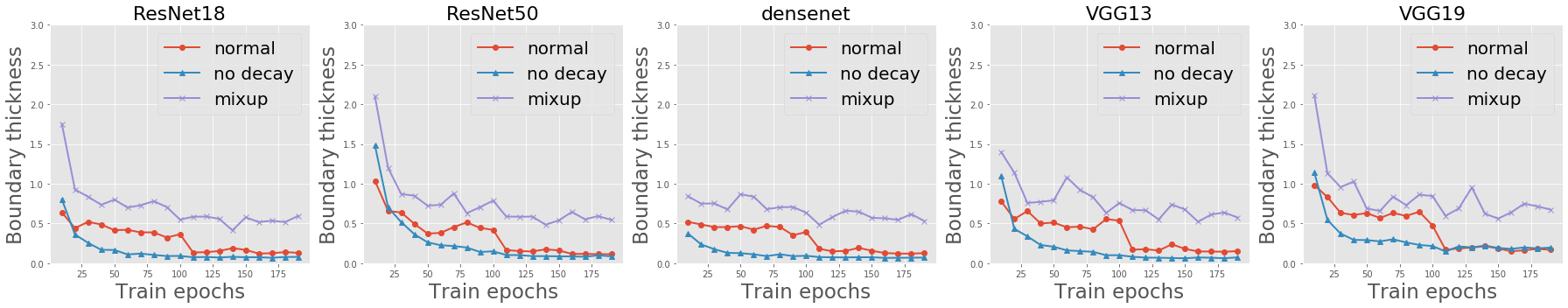}
    \subcaption{Results on CIFAR10 for $\alpha =0$ and $\beta =0.6$}
    \end{subfigure}
    
    \caption{{\bf Ablation study on different $\alpha$ and $\beta$.} Re-implementing the measurements in Figure \ref{fig:thickness_adv_direction} for different choices of $\alpha$ and $\beta$ in Eqn.\eqref{eqn:def_ij}.}
    \label{fig:different_alpha}
\end{figure}

In this subsection, we present an ablation study on the choice of hyper-parameters $\alpha$'s and $\beta$ in~\eqref{eqn:def_ij}. We show that the conclusions in Section \ref{sec:regularization} remain unchanged for a wide range of choices of $\alpha$ and $\beta$. See Figure \ref{fig:different_alpha}. From the results, we can see that the trend remains the same, i.e., mixup$>$normal training$>$training without weight decay. However, when $\alpha$ and $\beta$ become close to each other, the magnitude of boundary thickness also reduces, which is expected.  

\begin{remark}[Choosing the Best Hyper-parameters]
\normalfont
From Proposition \ref{prop:margin}, we know that the margin has particular values of the hyper-parameters $\alpha=0$ and $\beta=1$. Allowing different values of these hyper-parameters allows us the flexibility to better capture the robustness than margin. The best choices of these hyper-parameters might be different for different neural networks, and ideally one could do small validation based studies to tune these hyper-parameters, but our ablation study in this section shows that for a large regime of values, the exact search for the best choices is not required. We noticed, for example, setting $\alpha=0$ and $\beta=0.75$ works well in practice, and much better than the standard definition of margin that has been equated to robustness in past studies. 
\end{remark}

\begin{remark}[Choosing Asymmetric $\alpha$ and $\beta$]\label{rem:asymmetric_alpha}
\normalfont
We use asymmetric parameters $\alpha=0$ and $\beta>0$ mainly because, due to symmetry, the measured thickness when $(\alpha, \beta) = (0, x)$ is half in expectation of that when $(\alpha, \beta) = (-x, x)$. 
\end{remark}

We have discussed alternative ways of adversarial attacks to measure boundary thickness on sample pairs in Section \ref{sec:compare_different_thickness}. For completeness, we also do an ablation study for choice of hyper-parameters $\alpha$ and $\beta$ for this case. The results in Figure \ref{fig:different_alpha_adv}, and this study also reinforces the same conclusion -- that the particular choice of $\alpha$, $\beta$ matters less than the fact that they are not set to $0$ and $1$ respectively.

\begin{figure}[ht!]
    \centering
    
    \begin{subfigure}{0.98\textwidth}
    \includegraphics[width=\linewidth]{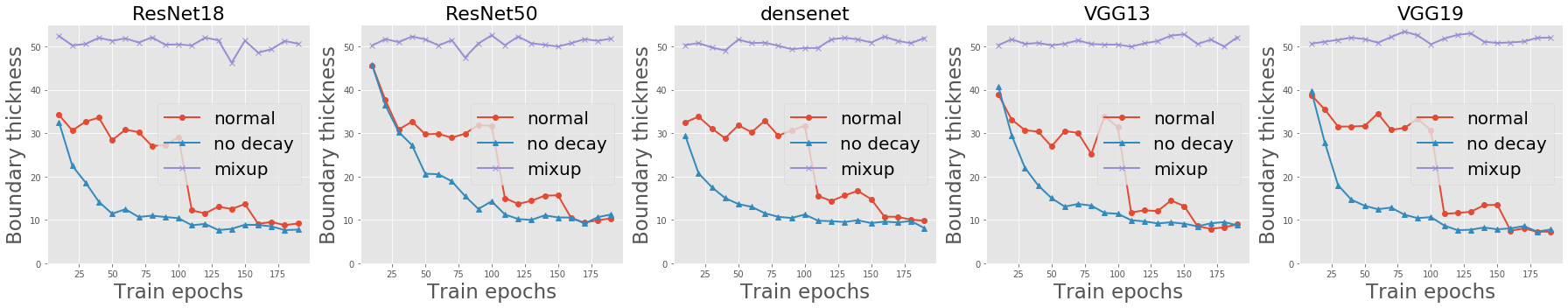}
    \subcaption{Results on CIFAR10 for $\alpha =0$ and $\beta =0.9$}
    \end{subfigure}
    
    \begin{subfigure}{0.98\textwidth}
    \includegraphics[width=\linewidth]{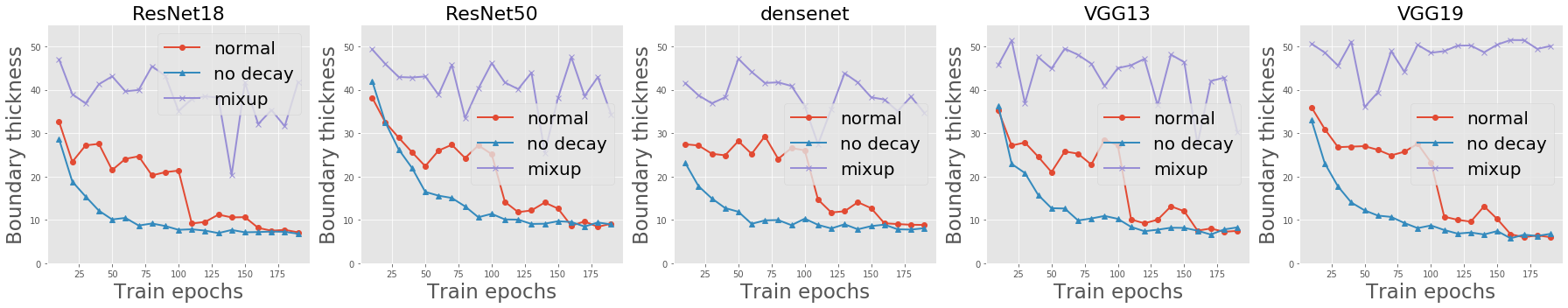}
    \subcaption{Results on CIFAR10 for $\alpha =0$ and $\beta =0.8$}
    \end{subfigure}
    
    \begin{subfigure}{0.98\textwidth}
    \includegraphics[width=\linewidth]{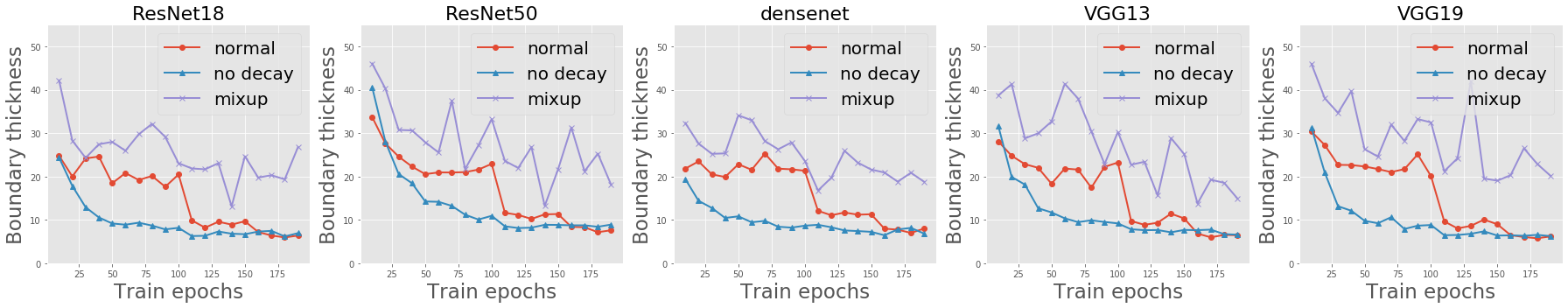}
    \subcaption{Results on CIFAR10 for $\alpha =0$ and $\beta =0.7$}
    \end{subfigure}
    
    \begin{subfigure}{0.98\textwidth}
    \includegraphics[width=\linewidth]{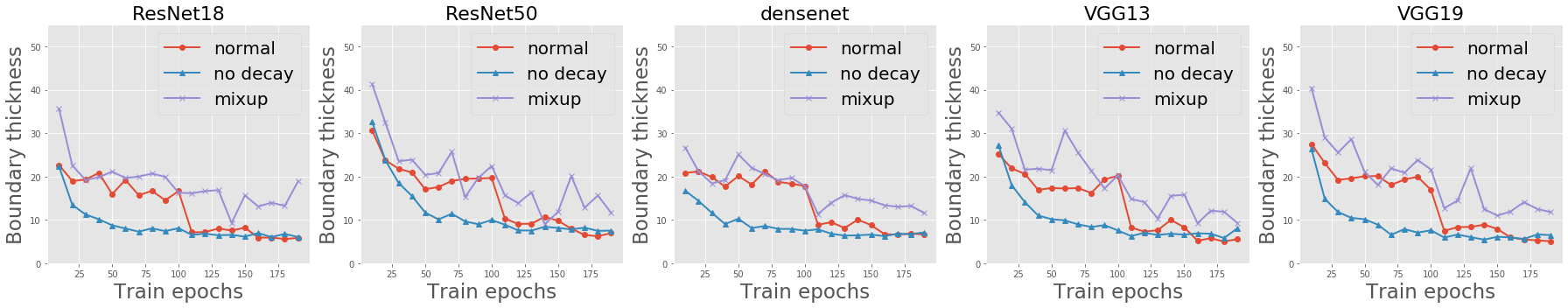}
    \subcaption{Results on CIFAR10 for $\alpha =0$ and $\beta =0.6$}
    \end{subfigure}
    
    \caption{{\bf Ablation study on different $\alpha$ and $\beta$ for thickness measured on random sample pairs.} Re-implementing the measurements in Figure \ref{fig:thickness_sample_pairs} for different choices of $\alpha$ and $\beta$ in Eqn.\eqref{eqn:def_ij}.}
    \label{fig:different_alpha_adv}
\end{figure}

\subsubsection{Additional datasets}

We repeat the experiments in Section \ref{sec:regularization} on two more datasets, namely CIFAR100 and SVHN. See Figure \ref{fig:different_data_sets}. In this figure, we used the same experimental setting as in Section \ref{sec:regularization}, except that we train with a different initial learning rate 0.01 on SVHN, following convention. We reach the same conclusion as in Section~\ref{sec:regularization}, i.e., that mixup increases boundary thickness, while training without weight decay reduces boundary thickness.

\begin{figure}[ht!]
    \centering
    
    \begin{subfigure}{0.98\textwidth}
    \includegraphics[width=\linewidth]{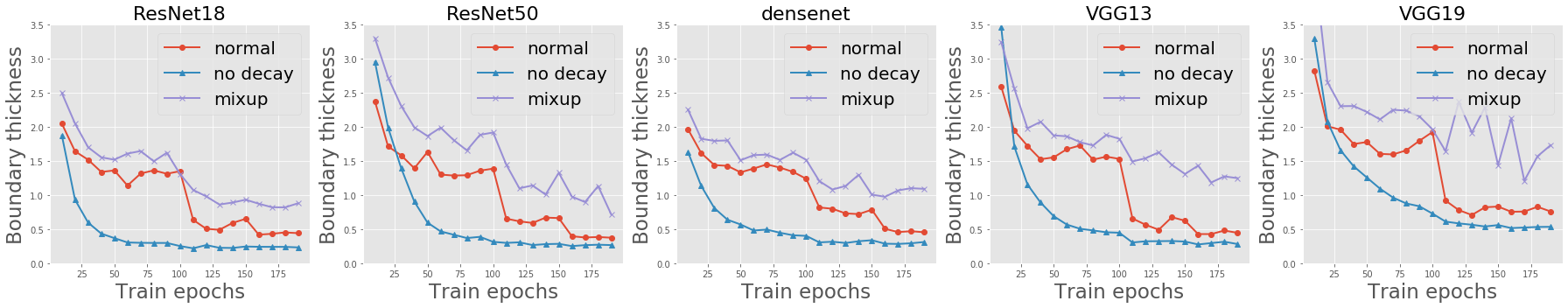}
    \subcaption{Results on CIFAR100}
    \end{subfigure}
    
    \begin{subfigure}{0.98\textwidth}
    \includegraphics[width=\linewidth]{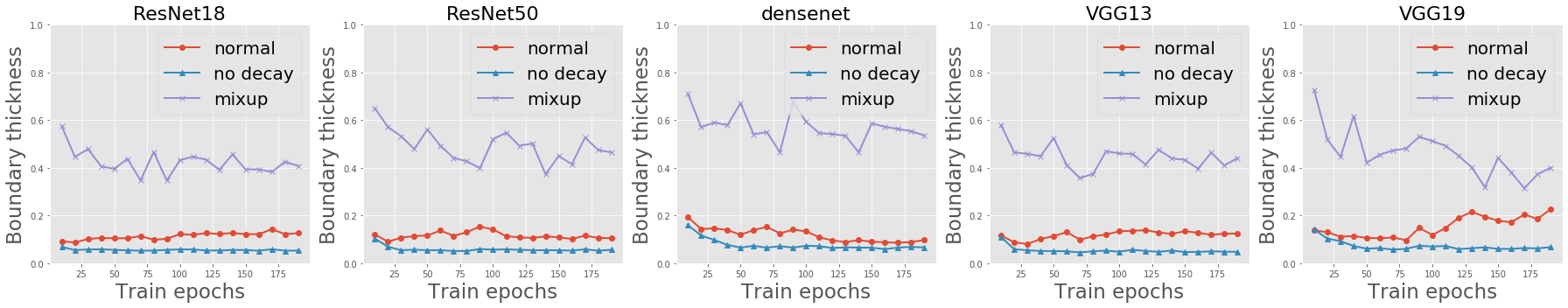}
    \subcaption{Results on SVHN}
    \end{subfigure}
    
    \caption{{\bf Ablation study on more datasets.} Re-implementing the measurements in Figure \ref{fig:thickness_adv_direction} on for two other datasets CIFAR100 and SVHN.}
    \label{fig:different_data_sets}
\end{figure}

\subsection{Visualizing neural network boundaries}\label{sec:visualization_mixup_boundary}

In this section, we show a qualitative comparison between a neural network trained using mixup and another one trained in a standard way without mixup. See Figure \ref{fig:decision_boundary_mixup}. In the left figure, we can see that different level sets are spaced apart, while the level sets in the right figure are hardly distinguishable. Thus, the mixup model has a larger boundary thickness than the naturally trained model for this setting.

For the visualization shown in Figure~\ref{fig:decision_boundary_mixup}, we use 17 different colors to represent the 17 level sets. The origin represents a randomly picked CIFAR10 image. The $x$-axis represents a direction of adversarial perturbation found using the projected gradient descent method \cite{madry2017towards}. The $y$-axis and the $z$-axis represent two random directions that are orthogonal to the $x$ perturbation direction. Each CIFAR10 input image has been normalized using standard routines during training, e.g., using the standard deviations $(0.2023, 0.1994, 0.2010)$, respectively, for the RGB channels, so the scale of the figure may not represent the true scale in the original space of CIFAR10 input images.

\begin{figure}[ht!]
\qquad\qquad\qquad\quad Mixup\qquad\qquad\quad\qquad Standard training\\
    \centering
    \includegraphics[width=0.28\textwidth]{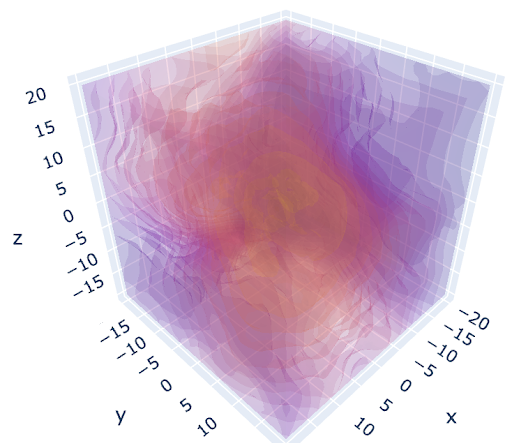}
    \includegraphics[width=0.30\textwidth]{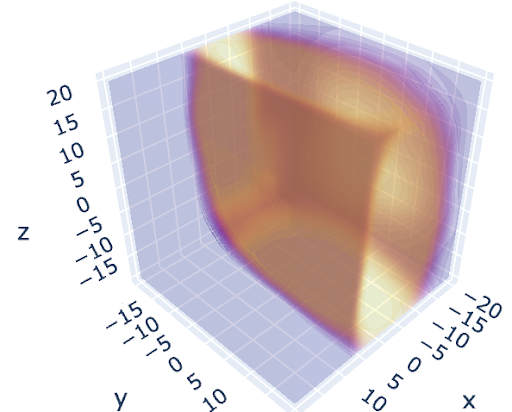}\hspace{2mm}
    \includegraphics[width=0.034\textwidth]{figs/color_bar.png}
    \caption{{\bf Visualization of mixup.} A mixup model has a thicker boundary than the same model trained using standard setting, because the level sets represented by different colors are more separate in mixup than the standard case. }
    \label{fig:decision_boundary_mixup}
\end{figure}

%% file: Appendix_adv_training.tex
\section{Additional Experiments on Adversarial Training}\label{sec:ablation}

In this section, we provide additional details and analyses for the experiments in Section \ref{sec:adv_training_thickness} on adversarially trained neural networks. We demonstrate that a thick boundary improves adversarial robustness for wide range of hyper-parameters, including those used during adversarial training and those used to measure boundary thickness.

\subsection{Details of experiments in Section \ref{sec:adv_training_thickness}}\label{sec:details_adv_experiment}

We use ResNet-18 on CIFAR-10 for all the experiments in Section \ref{sec:adv_training_thickness}. We first choose a standard setting that trains with learning rate 0.1, no learning rate decay, weight decay 5e-4, attack range $\epsilon = 8$ pixels, 10 iterations for each attack, and 2 pixels for the step-size. Then, for each set of experiments, we change one parameter based on the standard setting. We tune the parameters to achieve a natural training accuracy larger than 90\%. For the experiment on early stopping, we use a learning rate 0.01 instead of 0.1 to achieve 90\% training accuracy. We train the neural network for 400 epochs without learning rate decay to filter out the effect of early stopping. The results with learning rate decay and early stopping are reported in Section \ref{sec:adv_decay} which show the same trend. 

When measuring boundary thickness, we select segments on the adversarial direction, and we find the adversarial direction by using an $\ell_2$ PGD attack of size $\epsilon=$ 2.0, step size 0.2, and number of attack steps 20. %The hyper-parameters used to obtain the results in Figures \ref{fig:adv_thickness_all_1},\ref{fig:adv_thickness_all_2} are reported in Table \ref{tab:hyper_parameters}.

\begin{table}[ht!]
\small
\centering
\begin{tabular}{p{2.5cm}p{1.8cm}p{1.8cm}p{1.8cm}p{1.2cm}p{2.0cm}}
\hline
\hline
\ga Changed parameter & Learning rate & Weight decay & L1 & Cutout & Early stopping  \\
\hline
\gb Learning rate &  3e-3 &  5e-4   &   0   &   0   &   None  \\
\ga               &  1e-2 &  5e-4   &   0   &   0   &   None  \\
\gb               &  3e-2 &  5e-4   &   0   &   0   &   None  \\
\hline
\ga Weight decay  &  1e-1 &  0e-4   &   0   &   0   &   None  \\
\gb               &  1e-1 &  1e-4   &   0   &   0   &   None  \\
\hline
\ga L1            &  1e-1 &  0      &  5e-7 &   0   &   None  \\
\gb               &  1e-1 &  0      &  2e-6 &   0   &   None  \\
\ga               &  1e-1 &  0      &  5e-6 &   0   &   None  \\
\hline
\gb Cutout        &  1e-1 &  0      &   0   &   4   &   None  \\
\ga               &  1e-1 &  0      &   0   &   8   &   None  \\
\gb               &  1e-1 &  0      &   0   &  12   &   None  \\
\ga               &  1e-1 &  0      &   0   &  16   &   None  \\
\hline
\gb Early stopping&  1e-2 &  5e-4   &   0   &   0   &   50  \\
\ga               &  1e-2 &  5e-4   &   0   &   0   &   100  \\
\gb               &  1e-2 &  5e-4   &   0   &   0   &   200  \\
\ga               &  1e-2 &  5e-4   &   0   &   0   &   400  \\
\hline
\end{tabular}
\caption{{\bf Hyper-parameters in Section \ref{sec:adv_training_thickness}.} The table reports the hyper-parameters used to obtain the results in Figure \ref{fig:adv_thickness_all_1} and \ref{fig:adv_thickness_all_2} for adversarial training.\label{tab:hyper_parameters}}
\end{table}

Note that boundary thickness indeed increases with heavier regularization or data augmentation. See Figure \ref{fig:thickness_increases_regularization} on the thickness of the models trained with the parameters reported in Table \ref{tab:hyper_parameters}. %Also, in early stopping, a smaller epoch number implies heavier regularization.
\begin{figure}[ht!]
    \centering
    \includegraphics[width=0.19\textwidth]{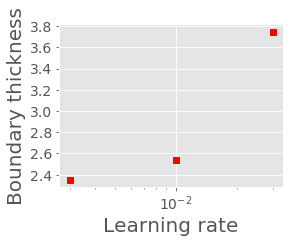}
    \includegraphics[width=0.19\textwidth]{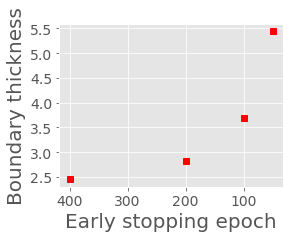}
    \includegraphics[width=0.19\textwidth]{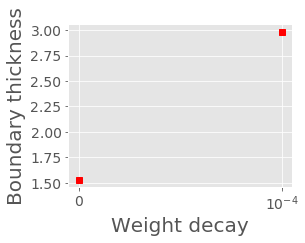}
    \includegraphics[width=0.19\textwidth]{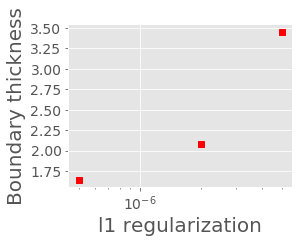}
    \includegraphics[width=0.19\textwidth]{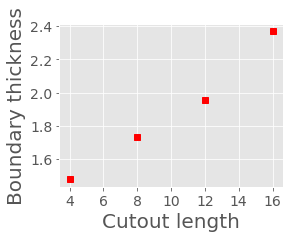}
    \caption{{\bf Regularization and data augmentation increase thickness.} The five commonly used regularization and data augmentation schemes studied in Section \ref{sec:adv_training_thickness} all increase boundary thickness.}
    \label{fig:thickness_increases_regularization}
\end{figure}

\subsection{Adversarial training with learning rate decay}\label{sec:adv_decay}

Here, we reimplement the experiments shown in Figure \ref{fig:adv_thickness_all_1} and \ref{fig:adv_thickness_all_2} but with learning rate decay and early stopping, which is reported by \cite{rice2020overfitting} to improve the robust accuracy of adversarial training. We still use ResNet-18 on CIFAR-10. However, instead of training for 400 epochs, we train for only 120 epochs, with a learning rate decay of 0.1 at epoch 100. The adversarial training still uses 8-pixel PGD attack with 10 steps and step size 2 pixel.

The set of training hyper-parameters that we use are shown in Table \ref{tab:parameter_adv_decay}. Similar to Figure \ref{fig:adv_thickness_all}, we tune hyper-parameters such that the training accuracy on natural data reaches 90\%. The results are reported in Figure \ref{fig:adv_thickness_all_decay}. We do not separately test early stopping because all experiments follow the same early stopping~procedure.

\begin{figure}
    \centering
     \begin{subfigure}{0.45\textwidth}
        \includegraphics[width=1.0\linewidth]{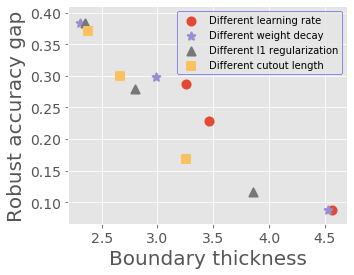}
        \subcaption{\label{fig:adv_thickness_all_decay_1}}
    \end{subfigure}
    \begin{subfigure}{0.45\textwidth}
        \includegraphics[width=1.0\linewidth]{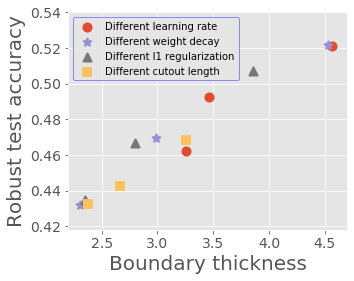}
        \subcaption{\label{fig:adv_thickness_all_decay_2}}
    \end{subfigure}
    \caption{{\bf Adversarial training with learning rate decay.} Reimplementing the experimental protocols that obtained Figure \ref{fig:adv_thickness_all} using learning rate decay. The hyper-parameters are shown in Table~\ref{tab:parameter_adv_decay}}
    \label{fig:adv_thickness_all_decay}
\end{figure}

\begin{table}
\small
\centering
\begin{tabular}{p{2.8cm}p{1.8cm}p{1.8cm}p{1.8cm}p{1.2cm}p{2.0cm}}
\hline
\hline
\ga Changed parameter & Learning rate & Weight decay & L1 & Cutout & Early stopping  \\
\hline
\gb Learning rate &  1e-2 &  5e-4   &   0   &   0   &   120  \\
\ga               &  3e-2 &  5e-4   &   0   &   0   &   120  \\
\gb               &  1e-1 &  5e-4   &   0   &   0   &   120  \\
\hline
\ga Weight decay  &  1e-1 &  0e-4   &   0   &   0   &   120  \\
\gb               &  1e-1 &  1e-4   &   0   &   0   &   120  \\
\ga               &  1e-1 &  5e-4   &   0   &   0   &   120  \\
\hline
\gb L1            &  1e-1 &  0      &  5e-7 &   0   &   120  \\
\ga               &  1e-1 &  0      &  2e-6 &   0   &   120  \\
\gb               &  1e-1 &  0      &  5e-6 &   0   &   120  \\
\hline
\ga Cutout        &  1e-1 &  0      &   0   &   4   &   120  \\
\gb               &  1e-1 &  0      &   0   &   8   &   120  \\
\ga               &  1e-1 &  0      &   0   &  12   &   120  \\
\hline
\hline
\end{tabular}
\caption{Hyper-parameter settings in Figure \ref{fig:adv_thickness_all_decay}.\label{tab:parameter_adv_decay}}
\end{table}

\subsection{Different choices of the hyper-parameters in measuring adversarially trained networks}\label{sec:ablation_adv_networks}

Here, we study the connection between boundary thickness and robustness under different choices of hyperparameters used to measure thickness. Specifically, we use three different sets of hyper-parameters to reimplement the experiments that obtained Figure \ref{fig:adv_thickness_all_1} and Figure \ref{fig:adv_thickness_all_2}. These parameters are provided in Table \ref{tab:parameter_adv_ablation}. The first row represents the base parameters used in Figure \ref{fig:adv_thickness_all_1} and Figure \ref{fig:adv_thickness_all_2}. Then, the second row changes $\beta$. The third and the fourth row change the attack size $\epsilon$ and step size. The changes in the hyper-parameters maintains our conclusion regarding the relationship between thickness and robustness. See Figure \ref{fig:adv_exp_ablation}.

\begin{figure}
    \centering
    \begin{subfigure}{0.45\textwidth}
        \includegraphics[width=\linewidth]{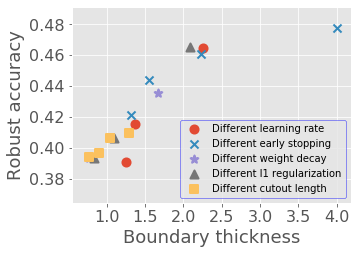}
        \subcaption{\label{fig:adv_exp_gap_1}}
    \end{subfigure}
     \begin{subfigure}{0.45\textwidth}
        \includegraphics[width=\linewidth]{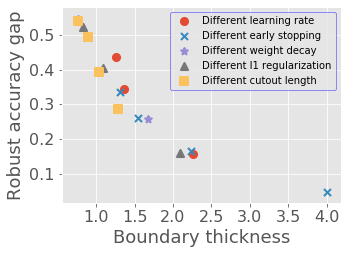}
        \subcaption{\label{fig:adv_exp_acc_1}}
    \end{subfigure}
    \begin{subfigure}{0.45\textwidth}
        \includegraphics[width=\linewidth]{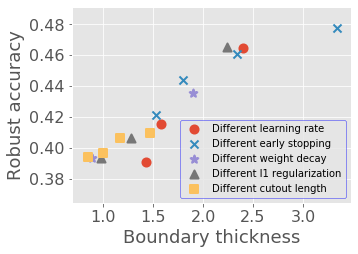}
        \subcaption{\label{fig:adv_exp_gap_2}}
    \end{subfigure}
     \begin{subfigure}{0.45\textwidth}
        \includegraphics[width=\linewidth]{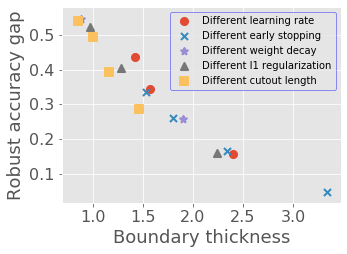}
        \subcaption{\label{fig:adv_exp_acc_2}}
    \end{subfigure}
    \begin{subfigure}{0.45\textwidth}
        \includegraphics[width=\linewidth]{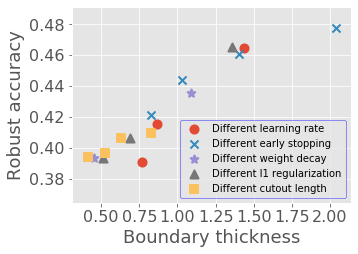}
        \subcaption{\label{fig:adv_exp_gap_3}}
    \end{subfigure}
     \begin{subfigure}{0.45\textwidth}
        \includegraphics[width=\linewidth]{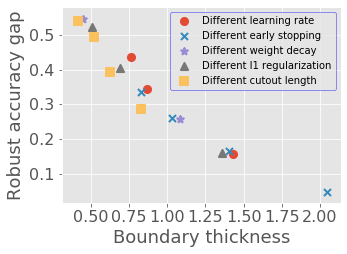}
        \subcaption{\label{fig:adv_exp_acc_3}}
    \end{subfigure}
    \caption{{\bf Ablation study on adversarially trained networks.} Re-implementing the measurements in Figure \ref{fig:adv_thickness_all} for three other different choices of $\alpha$, $\beta$ and the attack size $\epsilon$. The parameters are provided in Table \ref{tab:parameter_adv_ablation}. For all settings, robustness still increases with thickness.}
    \label{fig:adv_exp_ablation}
\end{figure}

\begin{table}
\centering
\begin{tabular}{p{3.0cm}p{1.2cm}p{1.2cm}p{1.2cm}p{2.5cm}p{1.2cm}}
\hline
\hline
\ga Figures & $\alpha$ & $\beta$ & Attack $\epsilon$ & Number of steps & Step size  \\
\hline
\gb Fig. \ref{fig:adv_thickness_all_1} and Fig. \ref{fig:adv_thickness_all_2} \qquad (the base setting in the main paper) &  0 &  0.75   &   2.0   &   20   &   0.2  \\
\ga Fig. \ref{fig:adv_exp_gap_1} and \ref{fig:adv_exp_acc_1}             &  0 &  0.5    &   2.0   &   20   &   0.2  \\
\gb Fig. \ref{fig:adv_exp_gap_2} and \ref{fig:adv_exp_acc_2}             &  0 &  0.75   &   1.0   &   20   &   0.1  \\
\ga Fig. \ref{fig:adv_exp_gap_3} and \ref{fig:adv_exp_acc_3}             &  0 &  0.75   &   0.6   &   20   &   0.06  \\
\hline
\hline
\end{tabular}
\caption{Hyper-parameter settings in Figure \ref{fig:adv_exp_ablation}.\label{tab:parameter_adv_ablation}}
\end{table}

\subsection{Additional experiment on comparing boundary thickness and margin}\label{sec:margin_vs_thickness_more}

Here, we further analyze how boundary thickness compares to the margin. 
In particular, we study what happens if we change the two hyper-parameters $\alpha$ and $\beta$ to the limit when boundary thickness becomes similar to margin, especially average margin shown in Figure \ref{fig:adv_thickness_all_3}. 
The purpose of this additional experiment is to study the distinction between boundary thickness and average margin, even when they become similar to each other.

We notice that, when reporting boundary thickness, apart from the ablation study, we always use parameters $\alpha=0$ and $\beta=0.75$. 
%As we have discussed in Remark \ref{rem:asymmetric_alpha}, this choice of using asymmetric parameters is not essential. 
When measuring average margin shown in Figure \ref{fig:adv_thickness_all_3}, as we mentioned in Section \ref{sec:thickness_beats_margin}, we approximately measure margin on the direction of adversarial examples, which effectively measures the boundary thickness when setting $\alpha=0$ and $\beta=1$. 
Thus, a reasonable questions is why the results of average margin in Figure \ref{fig:adv_thickness_all_3} are fundamentally different from boundary thickness, despite the small change only in $\beta$.

To answer this question, we empirically show the transition from $\beta = 0.9$ to $1.0$. See Figure \ref{fig:change_beta}. We see that a sudden phase-transition happens as $\beta $ gets close to $1.0$. This phenomenon has two implications. First, although boundary thickness reduces to margin with specific choices of parameters (in particular, $\alpha=0$ and $\beta=1$), it is different from margin when $\beta$ is not close to $1$. Second, for a large range of $\beta$, boundary thickness can distinguish robustness better than margin.
    
\begin{figure}
    \centering
    \includegraphics[width=.195\textwidth]{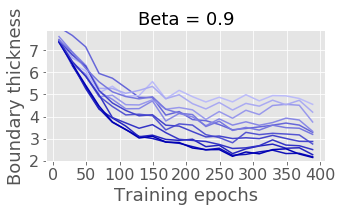}
    \includegraphics[width=.195\textwidth]{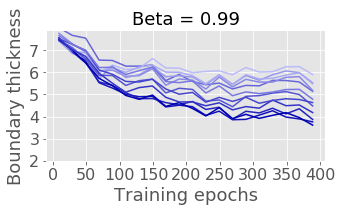}
    \includegraphics[width=.195\textwidth]{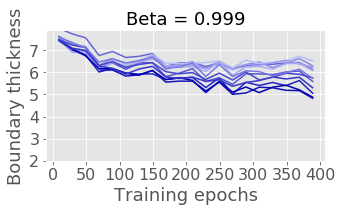}
    \includegraphics[width=.195\textwidth]{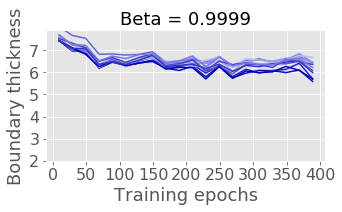}
    \includegraphics[width=.195\textwidth]{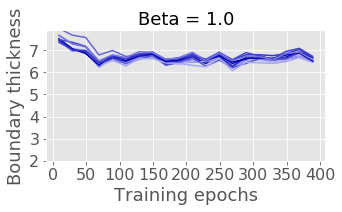}
    \caption{{\bf Boundary thickness versus margin.} Reimplementing procedures in Fig. \ref{fig:adv_thickness_all_3} with different $\beta$ when $\alpha=0$. There is a large range of $\beta$'s in which boundary thickness can measure robustness. 
    }
    \label{fig:change_beta}
\end{figure}

%% file: Appendix_noisy_mixup.tex
\section{Addition Experiments and Details of Noisy Mixup}\label{sec:details_experiments_noisy_mixup}

In this section, we first provide more details for the experiments on noisy mixup. Then, we provide some additional experiments to further compare noisy mixup and ordinary mixup.

\subsection{More details of noisy-mixup}\label{sec:details_noisy_mixup}

In the experiments, we use ResNet-18 on both CIFAR-10 and CIFAR-100. For OOD, we use the datasets from \cite{hendrycks2019using} and evaluate on 15 different types of corruptions in CIFAR10-C and CIFAR100-C, including noise, blur, weather, and digital corruption.

The probability to replace a clean image with a noise image is 0.5. 
The model is trained for 200 epochs, and learning rate decay with a factor of 10 at epochs 100 and 150, respectively. For both mixup and noisy mixup, we train using three learning rates $0.1, 0.03, 0.01$ and report the best results. The weight decay is set to be 1e-4, which follows the recommendation in \cite{zhang2017mixup}. 
For noisy mixup, the value on each pixel in a noise image is sampled independently from a uniform distribution in [0,1], and processed using the same standard training image transforms applied on the ordinary image~inputs.

When testing the robustness of the ordinary mixup and noisy mixup models, we used both black/white-box attacks and OOD samples. For white-box attack, we use an $\ell_\infty$ PGD attack with 20 steps. The attack size $\epsilon$ can take values in 8-pixel (0.031), 6-pixel (0.024) and 4-pixel (0.0157), respectively. The step size is 1/10 of the attack size $\epsilon$. For black-box attacks, we use ResNet-110 to generate the transfer attack. The other parameters are the same with the 8-pixel white-box attack.

When we measure the boundary thickness values of models trained using mixup and noisy mixup, we estimate each thickness value by repeated runs on 320 random samples. Then, we repeat this procedure 10 times and report both the average and three times the standard deviation in Figure \ref{fig:noisy_mixup}.

\subsection{Addition experiments on noisy-mixup}\label{sec:more_experiments_noisy_mixup}

Here, we provide two additional experiments on noisy mixup. 
In the first experiment, we want to study the effect of introducing an additional ``NONE'' class in noisy mixup.
In the second experiment, we justify the clean accuracy drop of noisy mixup as a consequence of insufficient model size.

In the first experiment, we focus on the additional ``NONE'' class introduced in noisy mixup. 
We notice that we attribute the improved robustness to the mixing of the ``NONE'' class with the original images in the dataset. 
However, due to this additional class, it is possible that the network can have improved robustness by learning to distinguish clean images from noise.
Thus, to separate the pure effect of having the additional ``NONE'' class in noisy mixup, we compare mixup and noisy mixup when mixup also has this additional class. 
Specifically, we measure ordinary mixup on CIFAR10 with the 11th class but only mix sample pairs within the first ten classes or within the 11th class. 
We call this method ``mixup-SEP''.
See the results in Table \ref{tab:sep_mixup}.
Note that the results of mixup and noisy mixup are the same with Table \ref{tab:noisy_mixup}.
From the results, we see that noisy mixup is more robust than mixup-SEP. Thus, we cannot attribute the improved robustness solely to the ``NONE'' class.

In the second experiment, we further study the clean accuracy drop of noisy mixup shown in Table~\ref{tab:noisy_mixup}. As we have mentioned, one reason for the clean accuracy drop is the tradeoff between clean accuracy and robust accuracy often seen in robust training algorithms. The supporting evidence is that adversarial training in the same setting using ResNet-18 on CIFAR100 only achieve 57.6\% clean accuracy.
Here, we analyze another potential factor that leads to the drop of clean accuracy, which is the size of the network. 
To study this factor, we change ResNet-18 to ResNet-50 and repeat the experiment. 
We report the results in Table~\ref{tab:ResNet50_mixup}.
Note that the results for ResNet-18 are the same with Table \ref{tab:noisy_mixup}.
Thus, we can see that the drop of clean accuracy reduces when using the larger ResNet-50 compared to ResNet-18.

\begin{table}
\small
\centering
\begin{tabular}{p{1.0cm}p{1.8cm}p{1.5cm}p{1.2cm}p{1.5cm}p{1.2cm}p{1.2cm}p{1.2cm}}
\hline
\hline
\ga Dataset & \multicolumn{1}{c}{Method} & \multicolumn{1}{c}{Clean} & \multicolumn{1}{c}{OOD} & Black-box  & \multicolumn{3}{c}{PGD-20} \\
\hline 
\gb &&&&& 8-pixel & 6-pixel & 4-pixel\\
\hline
\ga CIFAR10 &  \multicolumn{1}{c}{Mixup} & {\bf96.0}$\pm$0.1 &  78.5$\pm$0.4   &  46.3$\pm$1.4 & 2.0$\pm$0.1        &  3.2$\pm$0.1  &  6.3$\pm$0.1 \\
\gb  &  \multicolumn{1}{c}{Mixup-SEP} & 95.3$\pm$0.2 &  82.1$\pm$0.9   &  55.5$\pm$8.6 & 4.4$\pm$1.6        &  6.5$\pm$1.9  &  11.3$\pm$2.4 \\
\ga  & \multicolumn{1}{c}{Noisy mixup} & 94.4$\pm$0.2 &  {\bf 83.6}$\pm$0.3 &  {\bf 78.0}$\pm$1.0 & {\bf 11.7}$\pm$3.3  & {\bf 16.2}$\pm$4.2  & {\bf 25.7}$\pm$5.0 \\
 \hline
 \hline
\end{tabular}
\caption{{\bf Mixup-SEP.} Having an additional ``NONE'' class in mixup-SEP can help improve robustness of mixup, but it still cannot achieve the robustness of noisy mixup.
Results are reported for the best learning rate in [0.1, 0.03, 0.01].} \label{tab:sep_mixup}
\end{table}

\begin{table}
\small
\centering
\begin{tabular}{p{1.5cm}p{1.8cm}p{1.4cm}p{1.1cm}p{1.4cm}p{1.1cm}p{1.1cm}p{1.1cm}}
\hline
\hline
\ga Dataset, & \multicolumn{1}{c}{Method} & \multicolumn{1}{c}{Clean} & \multicolumn{1}{c}{OOD} & Black-box  & \multicolumn{3}{c}{PGD-20} \\
\gb Model &&&&& 8-pixel & 6-pixel & 4-pixel\\
\hline
\ga CIFAR100, &  \multicolumn{1}{c}{Mixup} & {\bf78.3}$\pm$0.8 &  51.3$\pm$0.4  & 37.3$\pm$1.1 &  0.0$\pm$0.0       &  0.0$\pm$0.0 &  0.1$\pm$0.0\\
%\cmidrule{2-7}
\gb ResNet-18 & \multicolumn{1}{c}{Noisy mixup} & 72.2$\pm$0.3 &  {\bf 52.5}$\pm$0.7 &  {\bf 60.1}$\pm$0.3 & {\bf 1.5}$\pm$0.2  & {\bf 2.6$\pm$0.1} & {\bf 6.7}$\pm$0.9 \\
\hline 
\ga CIFAR100, &  \multicolumn{1}{c}{Mixup} & {\bf79.3}$\pm$0.6 &  53.4$\pm$0.2  & 39.7$\pm$1.3 &  1.0$\pm$0.1       &  1.6$\pm$0.2 &  3.1$\pm$0.4\\
%\cmidrule{2-7}
\gb ResNet-50 & \multicolumn{1}{c}{Noisy mixup} & 75.5$\pm$0.5 &  {\bf 55.5}$\pm$0.3 &  {\bf 59.7}$\pm$1.1 & {\bf 4.3}$\pm$0.3  & {\bf 6.4$\pm$0.1} & {\bf 10.3}$\pm$0.0 \\
 \hline
 \hline
\end{tabular}
\caption{{\bf Clean accuracy drop of noisy mixup on CIFAR100.} The drop of clean accuracy can be mitigated by using a larger ResNet-50 network.
Results are reported for the best learning rate in [0.1, 0.03, 0.01]. \label{tab:ResNet50_mixup}}
\end{table}

%% file: Appendix_non_robust_feature.tex
\section{More Details of the Experiment on Non-robust Features in Section \ref{sec:connect_other_works}}\label{sec:mixup_compare}

In this part, we provide more details of the experiment on non-robust features. First, we discuss the background on the discovery in \cite{ilyas2019adversarial}. Using a dataset $\mathcal{D}_\text{train} = \{(x_i,y_i)\}_{i=1}^n$, $\mathcal{D}_\text{test}$ and a $C$-class neural network classifier $f$, a new dataset $\mathcal{D}'$ is generated as follows:

\begin{itemize}%[noitemsep,leftmargin=*]
    \item ({\bf attack-and-relabel}) Generate an adversarial example $x_i'$ from each training sample $x_i$ such that the prediction of the neural network $y_i' = \arg\max_{j\in\{1,2,\ldots,C\}} f(x_i')_j$ is not equal to $y_i$. Then, label the new sample $x_i'$ as $y_i'$. The target class $y_i'$ can either be a fixed value for each class $y_i$, or a random class that is different from $y_i$. In this paper, we use random target classes.
    \item ({\bf test-on-clean-data}) Train a new classifier $f'$ on the new dataset $\mathcal{D}' = \{(x_i',y_i')\}_{i=1}^n$, evaluate on the original clean testset $\mathcal{D}_\text{test}$, and obtain a test accuracy $\text{ACC}$.
\end{itemize}

The observation in \cite{ilyas2019adversarial} is that by training on the completely mislabeled dataset $\mathcal{D}' = \{(x_i',y_i')\}_{i=1}^n$, the new classifier $f'$ still achieves a high $\text{ACC}$ on $\mathcal{D}_\text{test}$. The explanation in \cite{ilyas2019adversarial} is that each adversarial example $x'$ contains ``non-robust features'' of the target label $y_i'$, which are useful for generalization, and $\text{ACC}$ measures the reliance on these non-robust features. The test accuracy $\text{ACC}$ obtained in this way is the \emph{generalization accuracy} reported in Figure \ref{fig:regularize}.

In Figure \ref{fig:regularize}, the $x$-axis means different epochs in training a source model. Each error bar represents the variance of non-robust feature scores measured in 8 repeated runs. Thus, each point in this figure represents 8 runs of the same procedures of a non-robust feature experiment for a different source network, and each curve in Figure \ref{fig:regularize} contains multiple experiments using different source networks, instead of a single training-testing round. It is interesting to see that source networks trained for different number of epochs can achieve different non-robust feature scores, which suggests that when the decision boundary changes between epochs, the properties of the non-robust features also change.

In the experiments to generate Figure \ref{fig:regularize}, we use a ResNet56 model as the source network, and a ResNet20 model as the target network. These two ResNet models are standard for classification tasks on CIFAR10. The source network is trained for 500 epochs, with an initial learning rate of 0.1, weight decay 5e-4, and learning rate decay 0.1 respectively at epoch 150, 300, and 450. When training with a small learning rate, the initial learning rate is set as 0.003. When training with mixup, the weight decay is 1e-4, following the initial setup in \cite{zhang2017mixup}. The adversarial attack uses PGD with 100 iterations, an $\ell_2$ attack range of $\epsilon$ = 2.0, and an attack stepsize of 0.4. 

\begin{remark}[Why Thick Boundaries Reduce Non-robust Features]
\normalfont
Our explanation on why a thick boundary reduces non-robust feature score is that a thicker boundary is potentially more ``complex''\footnote{Note that, although various complexity measures are associated with generalization in classical theory, and the inductive bias towards simplicity may explain generalization of neural networks \cite{de2019random,valle2018deep}, it has been pointed out that simplicity may be at odds with robustness\cite{nakkiran2019adversarial}.}. Then, in the attack-and-relabel step, the adversarial perturbations are generated in a relatively ``random'' way, independent of the true data distribution, making the ``non-robust features'' preserved by adversarial examples disappear. 
Studying the inner mechanism of the generation of non-robust features and the connection to boundary thickness is a meaningful future work.
\end{remark}